\documentclass{article}
\usepackage{kr}

\usepackage{times}
\usepackage{soul}
\usepackage{url}
\usepackage[hidelinks]{hyperref}
\usepackage[utf8]{inputenc}
\usepackage[small]{caption}
\usepackage{graphicx}
\usepackage{amsmath}
\usepackage{amsthm}
\usepackage{booktabs}
\usepackage{algorithm}
\usepackage{algorithmic}
\usepackage{amssymb}
\usepackage{thmtools}
\usepackage{thm-restate}
\usepackage{mathtools}
\usepackage{booktabs}
\usepackage{multirow}
\usepackage{array}
\usepackage{subcaption}
\usepackage{color}
\usepackage{comment}
\usepackage[switch]{lineno}
\usepackage{xspace}

\usepackage{ifthen}
\newboolean{TR} 
\setboolean{TR}{true}

\usepackage{xcolor}
\definecolor{lightblue}{HTML}{AADDFF}
\definecolor{lightred}{HTML}{FF5522}
\definecolor{lightpurple}{HTML}{DDAACC}
\definecolor{lightgreen}{HTML}{AAFFAA}
\definecolor{lightgray}{HTML}{DDDDDD}
\definecolor{strongyellow}{HTML}{FFE219}

\newcommand{\bcqname}[0]{\ensuremath{\mathsf{BCQ}}}
\newcommand{\bcqclass}[1]{\ensuremath{\bcqname_{#1}}}

\newcommand{\stripped}[1]{\mathsf{strip}(#1)}
\newcommand{\nowindows}[1]{\mathsf{wfree}(#1)}
\newcommand{\nocarry}[1]{#1}
\newcommand{\notime}[1]{\mathsf{tfree}(#1)}
\newcommand{\ground}{\mathsf{grnd}}
\newcommand{\noTimeGroundNoWindow}[2]{\notime{\tground_{#2}(#1)}}
\newcommand{\tground}{\mathsf{t}\ground}
\newcommand{\tfground}{\mathsf{tf}\ground}

\newcommand{\letterForLarsFragments}{\ensuremath{\mathbf{L}}\xspace}
\newcommand{\CL}[1]{\ensuremath{\letterForLarsFragments{}^+_{\mathsf{#1}}}\xspace}
\newcommand{\Lfull}{\CL{}}

\newcommand{\Lany}{\letterForLarsFragments}

\newcommand{\pred}[1]{\ensuremath{\mathsf{#1}}}

\newcommand{\poperator}[0]{\pred{bOpr}}

\RequirePackage{stmaryrd}
\newcommand{\freshPred}[1]{\llbracket #1\rrbracket}

\newcommand{\ptemp}[0]{\pred{bTmp}}
\newcommand{\ptempFull}[0]{\pred{beltTmp}}
\newcommand{\phightemp}[0]{\pred{high}}
\newcommand{\plowspeed}[0]{\pred{slow}}

\newcommand{\pbelt}[0]{\pred{belt}}

\newcommand{\pproblemId}[0]{\pred{incId}}

\newcommand{\pspeed}[0]{\pred{bSpeed}}
\newcommand{\perror}[0]{\pred{block}}
\newcommand{\pwarning}[0]{\pred{warn}}

\newcommand{\passign}[0]{\pred{assign}}

\newcommand{\pbrokengear}[0]{\pred{brkG}}

\newcommand{\poperatorFull}[0]{\pred{beltOperator}}

\newcommand{\pbrokengearFull}[0]{\pred{brokenGear}}
\newcommand{\pproblemIdFull}[0]{\pred{incidentId}}

\newcommand{\complclass}[1]{\ensuremath{\textsc{#1}}\xspace}
\newcommand{\ptimeclass}[0]{$\complclass{PTime}$}

\newcommand{\doubleexp}[0]{$\complclass{2ExpTime}$}

\newcommand{\doubleexpcomplete}[0]{$\complclass{2ExpTime}$-complete}
\newcommand{\nlogspace}[0]{$\complclass{NL}$}

\newcommand{\np}[0]{$\complclass{NP}$}

\newcommand{\pspace}[0]{$\complclass{PSpace}$}

\newcommand{\defeq}{\coloneqq}
\newcommand{\elars}{LARS$^+$\xspace}
\newcommand{\lars}{LARS}

\newcommand{\dom}[1]{\ensuremath{\mathsf{dom}(#1)}}

\newcommand{\tlinearg}[1]{\ensuremath{\vec{#1}}}
\newcommand{\timeline}[0]{\tlinearg{T}}

\newif\ifdraft\drafttrue
\newif\iffinal\finalfalse
\newif\ifextended\extendedfalse
\newif\ifsubmission\submissiontrue
\newif\ifdotikz\dotikzfalse
\newif\ifinlineref\inlinereffalse

\newcolumntype{C}[1]{>{\centering\let\newline\\\arraybackslash\hspace{0pt}}m{#1}}
\newcommand{\chasebench}[0]{\emph{ChaseBench}}
\newcommand{\deep}[0]{\emph{Deep100}}

\makeatletter
\newcommand*\bigcdot{\mathpalette\bigcdot@{.6}}
\newcommand*\bigcdot@[2]{\mathbin{\vcenter{\hbox{\scalebox{#2}{$\m@th#1\bullet$}}}}}
\makeatother

\newcommand{\nop}[1]{}

\newcommand{\funcStyle}[1]{\mathsf{#1}}

\newcommand{\rewrite}[1]{\ensuremath{\funcStyle{rew}(#1)}}
\newcommand{\rewriteT}[1]{\ensuremath{\funcStyle{rew}_{\timeline}(#1)}}

\newcommand{\arity}[1]{\ensuremath{\funcStyle{ar}(#1)}}

\newtheorem{theorem}{Theorem}
\newtheorem{lemma}{Lemma}

\newtheorem{proposition}{Proposition}

\theoremstyle{definition}

\newtheorem{definition}{Definition}
\newtheorem{example}{Example}

\newcommand{\Predicates}{\ensuremath{\cP}}

\newcommand{\Constants}{\ensuremath{\cC}}

\newcommand{\Variables}{\ensuremath{\cV}}
\newcommand{\Nulls}{\ensuremath{\cN}}

\newcommand{\Atoms}{\ensuremath{\cA}}

\newcommand{\head}[1]{\ensuremath{\mathsf{h}(#1)}}
\newcommand{\body}[1]{\ensuremath{\mathsf{b}(#1)}}

\renewcommand{\vec}[1]{{\mathbf{#1}}}
\newcommand{\tuple}[1]{\langle #1\rangle} 

\mathchardef\minus="2D

\newcommand{\intpr}{v} 

\newcommand{\window}{\ensuremath{\boxplus}}
\newcommand{\timeWindow}[1]{\ensuremath{\boxplus}^{#1}}

\newcommand{\cA}{\ensuremath{\mathcal{A}}}

\newcommand{\cC}{\ensuremath{\mathcal{C}}}

\newcommand{\cN}{\ensuremath{\mathcal{N}}}

\newcommand{\cP}{\ensuremath{\mathcal{P}}}

\newcommand{\cV}{\ensuremath{\mathcal{V}}}

\newcommand{\bbN}{\ensuremath{\mathbb{N}}}

\newenvironment{myitemize}
{\begin{list}{{\small$\bullet$}}{%
\setlength{\topsep}{1pt} 
\setlength{\leftmargin}{0pt} 
\setlength{\itemindent}{10pt}}} 
{\end{list}}

\newcounter{myenumctr}

\newcommand{\blank}{\text{\textvisiblespace}}
\newcommand{\firstPointTimeline}[1]{0}

\title{Chasing Streams with Existential Rules}

\author{Jacopo Urbani$^1$\and Markus Kr\"otzsch$^2$\and Thomas Eiter$^3$\\
\affiliations
$^1$Vrije Universiteit Amsterdam, Department of Computer Science, The Netherlands\\
$^2$Technische Universit\"at Dresden, Knowledge-Based Systems Group, Germany\\
$^3$Technische Universit\"at Wien, Institute of Logic and Computation, Austria\\
\emails
jacopo@cs.vu.nl,markus.kroetzsch@tu-dresden.de,eiter@kr.tuwien.ac.at}

\begin{document}

\maketitle

\begin{abstract}
We study reasoning with existential rules to perform query answering over
streams of data.  On static databases, this problem has been widely studied, but
its extension to rapidly changing data has not yet been considered. To
bridge this gap, we extend LARS, a well-known framework for rule-based stream
reasoning, to support existential rules. For that, we show how to
translate LARS with existentials into a semantics-preserving set of existential
rules.  As query answering with such rules is undecidable in general, we
describe how to leverage the temporal nature of streams and
present suitable notions of acyclicity that ensure decidability.
\end{abstract}

\section{Introduction}

Streaming data arises in many applications, fostered by the need of deriving
timely insights from emerging information and the inherent impossibility of
storing all available data~\cite{margara_streaming_2014}. Stream reasoning has
become a productive area of KR with many formalisms
\cite{anicic_ep-sparql:_2011,le-phuoc_native_2011,barbieri_c-sparql:_2010,TigerH16,DBLP:journals/datasci/DellAglioVHB17,Kharlamov+19:STARQL,walkega2019reasoning}.
This multiplicity is justified by the breadth of scenarios where stream
processing is useful. Many of the approaches are distinguished from classical
temporal reasoning, e.g., since data snapshots (\emph{windows}) play an
important role to reduce data volumes.

A well-known formalism in this space is LARS~\cite{lars}, which is a rule-based
language for stream reasoning that combines concepts from logic programming with
dedicated stream operators to express windows and temporal quantifiers. For
example, the LARS rule $r_1{:}\ \timeWindow{3} \Box \ptempFull(X,Y) \land
\phightemp(Y) \rightarrow \pwarning(X)$ issues a warning if the temperature on a
conveyor belt has been high for all ($\Box$) last three time points
($\timeWindow{3}$).

Another prominent field in KR are \emph{existential rules}, which are also used
as a basis for ontological models, especially in applications with large amounts
of data
\cite{BLMS11:decline,DBLP:journals/jair/GrauHKKMMW13,DBLP:conf/kr/GottlobLP14}.
As a simple example, the rule $r_2{:}\ \pbelt(X)\,{\rightarrow}\,\exists Y.
\poperatorFull(X,Y)$, expresses that every belt has an operator (even if
unknown). Existential quantification is central for ontologies and provides high
expressivity beyond plain Datalog \cite{KMR:ChasePower2019}. While reasoning is
well-known to be undecidable in general~\cite{DBLP:conf/icalp/BeeriV81}, many
well-behaved language fragments and practical implementations exist
\cite{bench_chase,vlog-chase,DBLP:journals/pvldb/BellomariniSG18}. In
particular, the Datalog$^{\pm}$ approach
(Cal{\`{\i}} et
al.\ \citeyear{DBLP:conf/lics/CaliGLMP10,DBLP:journals/ws/CaliGL12};
Gottlob et al.\ \citeyear{DBLP:conf/kr/GottlobLP14,DBLP:conf/rweb/GottlobMP15})
turned out to be fruitful.

Until now, however, these areas have not been combined, and stream reasoning
approaches do not support existential rules.  Even for logic-based ontology
languages in general, solutions only seem to exist for specific cases where
queries are rewritable
\cite{Kharlamov+19:STARQL,DBLP:journals/amcs/Kalayci0CRXZ19}. As a consequence,
it is often unclear how existing ontological background knowledge can be used in
stream reasoning.

Additionally, the lack of existential quantification prevents useful
modelling techniques for stream analysis. In particular, existential
quantification can be used to represent temporal events,
possibly spanning multiple time points, or to track unknown individuals.  For
instance, it can be used to create a new incident ID if the temperature on a
belt is high for too long, or to track a not-yet-recognized object within a
bounding box in a video stream. Notice that while in principle events could be
modeled without value invention, i.e., using ad-hoc relations, doing so would
put an upper bound to the number of possible events which might be undesirable
as the future stream is typically unknown. A similar argument applies to the
example above about objects within bounding boxes: it is arguably more natural
to introduce new values and treat them as first-class individuals.

With this motivation in place, we developed an extension of existential rules
with LARS-based temporal quantifiers called \elars. Due to undecidability of
query answering with existential rules, our objective are decidable fragments,
with the following contributions:

\begin{myitemize}
\itemsep=0pt
 \item We introduce \elars as an existential stream
    reasoning language with a model-theoretic semantics.

\item We give a semantics-preserving transformation from \elars to existential
    rules to allow query answering.
Doing so allows us to exploit existing decidability results, but these are
limited in their use of time. We thus present  \emph{time-aware
extensions of acyclicity notions}\/ for \elars{} programs.

\item Initial
    experiments suggest that our
    method is promising.%
    \footnote{Source code is
      at~\url{https://github.com/karmaresearch/elars}%
\ifthenelse{\boolean{TR}}{%
.}{%
; for a longer version of this paper see \cite{tr}.}}
\end{myitemize}

\section{\elars}\label{sec_elars}

Currently, LARS and
DatalogMTL~\cite{10.5555/3298239.3298397,walkega2019reasoning} are popular for
rule-based reasoning over data streams. While we focus on LARS, some of our work
may be adapted to DatalogMTL.

To cope with big data volumes, \lars{} allows one to restrict streams to data
snapshots (i.e., substreams) taken by generic window operators $\window$.
Typically, windows are used to consider only the knowledge in the most recent
past, but this is not enough to avoid a complexity explosion or even
undecidability that could arise from reasoning over an indefinite future. To
overcome this problem, it is common in this domain to restrict future
predictions up to a horizon of interest $h$, which is moved forward
indefinitely.

Our language \elars{} can be viewed as an extension of existential rules with
temporal features of LARS. In the choice of data-snapshot operators, we take
inspiration from \emph{plain LARS}, which is a LARS fragment that is apt for
efficient implementation~\cite{laser}.

\smallskip

\noindent \textbf{Syntax\;} We consider a two-sorted logic with abstract elements and
the natural numbers $\mathbb{N}$ as time points.  We assume infinite sets
$\Variables_A$ of \emph{abstract variables}, $\Variables_T$ of \emph{time
variables}, $\Nulls$ of \emph{labelled nulls}, and $\Constants$ of
\emph{constants} that are mutually disjoint and disjoint from $\mathbb{N}$.
\emph{Abstract terms} (resp.\ \emph{time terms})  are elements of
$\Variables_A\cup\Nulls\cup\Constants$ (resp., $\Variables_T\cup\mathbb{N}$).

Predicates $p$ are from a set $\Predicates$ of predicates and have arity
$\arity{p}\geq 0$, with each position typed (abstract or time sort).  A
\emph{normal atom} is an expression $p(\vec{t})$,
$\vec{t}=t_1,\ldots,t_\arity{p}$, where $t_i$ is a term of proper sort.  An
\emph{arithmetic atom} has the form $t_1\leq t_2$ or $t_1=t_2+t_3$ for time
terms $t_1,t_2,t_3$. The set of all \emph{atoms} (normal and arithmetic) is
denoted $\Atoms$.  For an atom $\alpha$ (or any other logical expression
introduced below), the \emph{domain} $\dom{\alpha}$ of $\alpha$ is the set of
all terms in $\alpha$; we write $\alpha[\vec{x}]$ to state that
$\vec{x}=\dom{\alpha}\cap(\Variables_A\cup\Variables_T)$; and we say that
$\alpha$ is \emph{ground} if it contains no variables.

A predicate $p\,{\in}\,\Predicates$ is \emph{simple} if it has no position of
time sort, while an atom is \emph{simple} if it normal and has a simple
predicate. A \emph{\elars atom} $\alpha$ has the form
\begin{linenomath}
\begin{align}
\alpha \defeq a \mid b \mid @_T b \mid \window^n @_T b \mid \window^n \Diamond b\mid \window^n \Box b
\label{eq_elars_atom}
\end{align}
\end{linenomath}
\noindent where $a$ is an arithmetic atom, $b$ is either a null-free
simple atom or $\top$ (which holds true at all times),
$T$ is a time term, and $n\in\mathbb{N}$.
\emph{Window operators} $\window^n$ restrict attention
back to $n$ time points in the past,
and $@_T$ (resp.\ $\Box$, $\Diamond$) indicates that a formula holds at time $T$
(resp., \emph{every}, {\em some} time point).

Arithmetic atoms do not depend on time, whereas atoms $@_T b$ refer to a
specific time $T$.  All other \elars atoms are interpreted relative to some
current time point.  Simple atoms $b$ can equivalently be written as $\window^0
\Diamond b$ or as $\window^0 \Box b$.

\begin{definition}\label{def_elars_rules} A \emph{\elars rule} is an expression
    of the form
\begin{linenomath}\begin{align}\label{eq:erule}
r=\Box\forall\vec{x},\vec{y}.(B[\vec{x},\vec{y}]\rightarrow\exists\vec{v}.H[\vec{y},\vec{v}])
\end{align}\end{linenomath}
\noindent where $\vec{x}$, $\vec{y}$, and $\vec{v}$ are mutually disjoint sets
of variables, and $\vec{v}$ contains only abstract variables; the body
$B[\vec{x},\vec{y}]$ is a conjunction of \elars atoms; and the head
$H[\vec{y},\vec{v}]$ is a conjunction of atoms of the form $b$ or $@_T b$ in
(\ref{eq_elars_atom}).  We set $\body{r}\defeq B$ and $\head{r}\defeq H$, and we
usually omit the leading $\Box$ and universal quantifiers when writing rules.

A \emph{\elars program} is a finite set of \emph{\elars rules}; we denote  the
set of all such programs by \Lfull.  \end{definition}

\noindent \textbf{Semantics\;} Like for LARS, the semantics of \elars is based on streams.
Formally, a \emph{stream} $S=(\timeline,\intpr)$ consists of a timeline
$\timeline=[0,h]\subset\bbN$ and an \emph{evaluation function} ${\intpr : \bbN
\to 2^\Atoms}$ such that, for all $t\in\bbN$, $\intpr(t)$ is a set of ground normal
atoms and $\intpr(t)=\emptyset$ if $t\notin\timeline$. We call $S$ a
{\em data stream} if only extensional atoms occur in $S$, i.e., atoms with
designated predicates not occurring in rule heads.  Given $n\in\bbN$ and
$t\in\timeline$, we write $w_n(S,t)$ for the stream $([0,t],\intpr')$ where for
any $t'\in \mathbb{N}$, $\intpr'(t')=\intpr(t')$ if $t-n\leq t'\leq t$, and
$\intpr'(t)=\emptyset$ otherwise; we call $w_n(S,t)$ a \emph{window}\/ of size
$n$ on $S$ at $t$.

Models of \elars are special streams.  For a stream $S=(\timeline,\intpr)$, a
simple ground atom $b$, and $t,t',n\in\bbN$, we write:

\begin{tabbing}
\qquad\=$S,t \models \Diamond b \,/\,  \Box b$~~~~\=\+\kill
$S,t \models b$~~~ if $b\in\intpr(t)$, \quad $S,t\models @_{t'} b$~~~ if $S,t'\models b$,\\[3pt]
$S,t \models \Diamond b \,/\,  \Box b$ \> if $S,t''\models b$ for some /  all $t''\in\timeline$,\\[3pt]
$S,t \models \window^n \beta$ \> if $w_n(S,t),t\models\beta$.
\end{tabbing}
Further, $S,t\models\top$ holds for all $t \in \timeline$ and $S,t\models a$
for all ground arithmetic atoms $a$ that express a true relation on $\mathbb{N}$.
To define satisfaction of rules on a stream $S$ at time point $t$, we
introduce the auxiliary notion of
\emph{$\timeline$-match} $\sigma$ for a set $C$ of atoms on
$S$ and $t$ as a sort-preserving mapping from the variables of $C$ to
terms, s.t.\ (i) each time variable $X$ is mapped to $\timeline$
($X\sigma \,{\in}\,\timeline$)
and (ii) $S,t \models \alpha\sigma$ for each $\alpha\,{\in}\,C$.

\begin{definition}\label{def_elars_rule_sem}
A \elars rule $r$
as in \eqref{eq:erule} is \emph{satisfied} by a stream $S=(\timeline,\intpr)$, written $S\models r$,
if either (i) $\head{r}$ contains some time point $t \notin \timeline$
(i.e., ignore inference out of scope),
or (ii) for all $t\in\timeline$, every $\timeline$-match
$\sigma$ of $\body{r}$ on $S$ and $t$
is extendible to a $\timeline$-match $\sigma'\supseteq\sigma$ of $\body{r}\cup\head{r}$ on $S$ and $t$.

A program $P\in\Lfull$ is satisfied by $S$, written $S\models P$, if $S\models
r$ for all $r\in P$.  A data stream $D=(\timeline',\intpr')$ is satisfied by
$S$, written $S\models D$, if $\timeline'\subseteq\timeline$ and
$\intpr'(t)\subseteq\intpr(t)$ for all $t\in\timeline'$.  We then call $S$ a
\emph{model} of $P$ resp.~$D$.
\end{definition}

\begin{example} Consider the data stream $D \,{=}\, ([0,9],\intpr)$, where
$\intpr(t) \,{=}\, \{ \pbelt(b_1),\phightemp(90), \ptempFull(b_1,tmp(t))\}$
for each $t\,{\in}\,[0,9]$, where $tmp(t) \,{=}\, 90$ if $t \,{\leq}\, 4$ and
$tmp(t) \,{=}\, 70$ otherwise. Then any
model $S$ of the rules $r_1,r_2$ in Section~1 and $D$ fulfills
$S,4 \,{\models}\, \pwarning(b_1)
\,{\land}\, \poperatorFull(b_1,v)$ for some constant or null $v$.
Similarly $S,5 \,{\models}\, \poperatorFull(b_1,v')$ for some constant or
null $v'$ while $S,5 \,{\models}\,  \pwarning(b_1)$
may fail.
\end{example}

\newcommand{\titleQuerySection}{Query Answering with \elars}
\section{\titleQuerySection}\label{sec_elars-to-tgds}

The query answering problem in \elars{}  is as follows.

\begin{definition}\label{def:bcq}
A \emph{\elars Boolean Conjunctive Query (BCQ)} $q$ has the form
$\exists\vec{x}.Q[\vec{x}]$, where $Q$ is a conjunction of \elars atoms.  A
stream $S=(\timeline,\intpr)$ satisfies $q$ at time $t$, written $S,t\models q$,
if some $\timeline$-match $\sigma$ of $Q$ on $S$ and $t$ exists. A program $P
\in \Lfull$ and data stream $D$ \emph{entail} $q$ at time $t$, written
$P,D,t\models q$, if $S,t\models q$ for every model $S$ of $P$ and $D$.
\end{definition}

For instance, a BCQ could be $\exists X. \window^5 \pwarning(X)$, which asks if
there has been a warning over the same belt in the last 5 time points. To solve
BCQ answering with \elars, we propose a consequence-preserving rewriting
$\rewrite{\cdot}$  to existential rules with a time sort. This rewriting is
useful because it will allow us to exploit known results for existential rules,
e.g., acyclicity notions~\cite{DBLP:journals/jair/GrauHKKMMW13}.

Our proposed rewriting of $P$ into $\rewrite{P}$ has 5 steps:

\noindent (1) Each atom $\window^n \Diamond p(\vec{t})$ is replaced by
$\window^n @_T\, p(\vec{t})$, where $T$ is a fresh variable used only in one
atom.

\noindent (2) For any simple predicate $p$, we add auxiliary predicates
$\freshPred{\window\Box\, p}$ and $\freshPred{\window @\, p}$ of arity
$\arity{p}{+}2$ resp.\ $\arity{p}{+}3$.  Intuitively, $\freshPred{\window\Box\,
p}(\vec{t},n,C)$ and $\freshPred{\window @\, p}(\vec{t},n,T,C)$ mean that
$\window^n \Box p(\vec{t})$ and $\window^n @_T\, p(\vec{t})$ hold at time $C$,
respectively.

\noindent (3) Using a fresh variable $C$ to represent the current time,
we rewrite non-arithmetic atoms $\alpha$ in $P$ (where $\top$ is $\top\!()$) to
\begin{linenomath}
\begin{align*}
\rewrite{\alpha} =
\left\{
\begin{array}{ll}
  \freshPred{\window\Box\,p}(\vec{t},0,C) &  \text{if } \alpha = p(\vec{t}),\\
 \freshPred{\window\Box\,p}(\vec{t},0,T) &  \text{if } \alpha = @_T\, p(\vec{t}),\\
 \freshPred{\window\Box\,p}(\vec{t},n,C) & \text{if } \alpha = \window^n \Box p(\vec{t}),\\
 \freshPred{\window @\, p}(\vec{t},n,T,C) & \text{if } \alpha = \window^n @_T\, p(\vec{t})
\end{array}\right.
\end{align*}
\end{linenomath}

\noindent (4) We add $\freshPred{\window \Box\,\top}(0,C)$ in rule bodies
not containing $C$.

\noindent (5) For every predicate $p$
(including $\top$), we add the following rules to $P$, where $\vec{X}$ is a list of variables
of length $\arity{p}$ and
$m = \max (0, n \mid \window^n$ occurs in $P)$:
{
\begin{linenomath}
{\setlength{\belowdisplayskip}{4pt}%
\allowdisplaybreaks%
\begin{align}%
0\,{\leq}\,C & \to \freshPred{\window\Box\,\top}(0,C)\label{rul_top}\\
\freshPred{\window\Box\,p}(\vec{X},0,0) &\to \freshPred{\window\Box\,p}(\vec{X},m,0) \label{rul_boxinit}\\
\!\!\freshPred{\window\Box\,p}(\vec{X},N',C)\wedge N'{=}N{+}1 &\to \freshPred{\window\Box\,p}(\vec{X},N,C) \label{rul_boxshift}
\end{align}%
}%
\setlength{\abovedisplayskip}{0pt}%
\begin{align}%
\begin{split}
 \!\!\freshPred{\window\Box\,p}(\vec{X},N,C) \wedge N'{=}N&{+}1 \wedge N'{\leq}m \wedge
 C'{=}C{+}1
\\[-3pt]
{}\wedge \freshPred{\window\Box\,p}(\vec{X},0,C')
	&\to \freshPred{\window\Box\,p}(\vec{X},N',C')
\end{split}\label{rul_boxexpand}\\
\freshPred{\window\Box\,p}(\vec{X},0,C) &\to \freshPred{\window @\, p}(\vec{X},0,C,C) \label{rul_boxinitat}\\
\begin{split}
\!\!\!\freshPred{\window @\, p}(\vec{X},N,T,C) \wedge N' &{\leq} m
  \wedge N'{=}N{+}1 \\[-3pt]
{}\wedge I\,{\leq}\,1\wedge C'{=}C{+}I &\to\freshPred{\window @\, p}(\vec{X},N'\!,T,C')
\end{split}\label{rul_boxexpandat}%
\end{align}%
\end{linenomath} %
}%

We rewrite a \elars BCQ $\exists\vec{x}.Q$ and time point $t$ similarly to
\mbox{$\rewrite{\exists\vec{x}.Q,t}=\exists\vec{x}.\bigwedge_{\alpha\in Q}
\rewrite{\alpha}\wedge C{\leq}t\wedge t{\leq}C$} (treating atoms $\window^n
\Diamond p(\vec{t})$ as before), and a stream $S\,{=}\,(\timeline,\intpr)$ to
facts $\rewrite{S}=\{\freshPred{\window\Box\,p}(\vec{t},0,s)\mid
p(\vec{t})\in\intpr(t), t\in\timeline \}$.
\begin{example}
We illustrate the rewriting on $r_1$. Step (2) creates predicates
$\freshPred{\window\Box \ptempFull}$, $\freshPred{\window\Box \phightemp}$,
and $\freshPred{\window\Box \pwarning}$ and Step (3) the rule
\mbox{$\freshPred{\window\Box \phightemp}(X,Y,3,C)\land$} \mbox{$\freshPred{\window\Box \phightemp}(Y,0,C) \to \freshPred{\window\Box
        \pwarning}(X,0,C)$}. Step (5) adds auxiliary
        rules to implement the semantics; e.g.,
rule (6) ensures that ``$\window\Box$''-facts survive across time
        points, say if $\freshPred{\window\Box \ptempFull}(a,b,0,6)$ and
        $\freshPred{\window\Box \ptempFull}(a,b,2,5)$ hold, then
 $\freshPred{\window\Box \ptempFull}(a,b,3,6)$ should hold as well.
\end{example}
Let us denote by $P'\models_\timeline q'$ entailment of a BCQ $q'$ from
existential rules $P'$ with timeline $\timeline$, which is defined using
$\timeline$-matches as $P,D,t \models q$ but disregarding $D$ and $t$. Then:

\begin{restatable}{theorem}{thmLARSplusIntoExrules}\label{theo_elars-in-exrules}
For any $P\,{\in}\,\Lfull$, BCQ $q$, data stream $D$ on
  $\timeline$, and $t{\in}\timeline$ holds
$P,D,t\,{\models}q$ iff
      $\rewrite{P}\cup\rewrite{D}{\models_\timeline}\rewrite{q,t}$.
\end{restatable}

Theorem~\ref{theo_elars-in-exrules} is important as it allows us to implement
BCQ answering in \elars{} using existential rule engines, e.g.,
GLog~\cite{glog}; arithmetic atoms over $\timeline$ can be simulated with
regular atoms: simply add the set $\rewrite{\timeline}$ of all true instances of
arithmetic atoms in $P$ over $\timeline$ and view
$\rewrite{P}\cup\rewrite{D}\cup\rewrite{\timeline}$ as a single-sorted theory.

\newcommand{\titleComplexity}{Decidability}
\section{\titleComplexity}\label{sec_tlwa}

As BCQ entailment over existential rules is undecidable, we desire that the
rewriting $\rewrite{\cdot}$ falls into a known decidable fragment. Such may be
defined by \emph{acyclicity conditions}\/
\cite{DBLP:journals/jair/GrauHKKMMW13}, which ensure that a suitable
\emph{chase}, which is a versatile class of reasoning algorithms for existential
rules \cite{bench_chase} based on ``applying'' rules iteratively, will terminate
over a given input. We use a variant of the \emph{skolem chase}
\cite{Marnette09:superWA}, using nulls instead of skolem terms (aka
\emph{semi-oblivious chase}), extended to the time sort%
\ifthenelse{\boolean{TR}}{}{(see ~\cite{tr})}.

Conditions like the canonical {\em weak acyclicity (WA)}\/
\cite{fagin_data_2005} ensure in fact \emph{universal termination}, i.e., chase
termination for a given rule set over all sets of input facts.  We can thus
apply such criteria to $\rewrite{P}$ (viewed as single-sorted theory) while
ignoring $\rewrite{D}$ and $\rewrite{\timeline}$. Universal termination may
here be seen as an analysis that disregards time. To formalise this, let
$\stripped{P}$ result from $P$ by deleting all arithmetic atoms, window
operators, and temporal quantifiers, and let $\mathbf{CT}$ and $\mathbf{WA}$ be
the classes of all rule sets on which the skolem chase universally terminates
and of all weakly acyclic rule sets, respectively. Then:

\begin{restatable}{theorem}{thmAcycStripped}\label{theo_acyc_stripped}
For any $P\in\Lfull$, we have
(i) $\stripped{P}\,{\in}\,\mathbf{CT}$ iff $\rewrite{P}\,{\in}\,\mathbf{CT}$
and
(ii) $\stripped{P}\,{\in}\,\mathbf{WA}$ iff $\rewrite{P}\,{\in}\,\mathbf{WA}$.
\end{restatable}

Analogous results hold for elaborated acyclicity notions
\cite{DBLP:journals/jair/GrauHKKMMW13}.  Notably, we can check acyclicity on the
simpler rule set $\stripped{P}$.  With WA as a representative notion, we let
$\CL{LWA}\,{=}\, \{ P\ \in \Lfull \mid \stripped{P} \in \mathbf{WA}\}$.

While easy to check, universal termination also considers situations that are
impossible on properly encoded streams.

\begin{example}\label{ex_univTermOverestimate} Consider
    $P\,{=}\,\{@_T\,p(X,Y)\wedge T'{=}T{+}1\to \exists V.@_{T'}\,p(Y,V)\}$.  The
    skolem chase on $\rewrite{P}\cup\rewrite{D}\cup\rewrite{\timeline}$
    terminates on all $\timeline$ and $D$, but not universally for non-standard
timelines where e.g.,\ $0=0\,{+}\,1$ holds.  That is, reasoning with $P$ always
terminates despite $P\not\in \CL{LWA}$.
\end{example}
We thus introduce \emph{time-aware acyclicity}, which retains relevant temporal
information instead of working with $\stripped{P}$ only.  First, to simplify
$P$, we fix a fresh time variable $N$ and replace all \elars atoms in all rules
as follows:

\vspace{-0.5\baselineskip}

\begin{linenomath}
\hspace*{-20pt}
\begin{tabular}{@{}l@{~}l@{}}
\parbox[b]{0.23\textwidth}{
\begin{align}
 p(\vec{t}) &\mapsto @_N\, p(\vec{t})\label{eq_tlwa_simple}\\
 \window^n \Box p(\vec{t}) &\mapsto @_N\, p(\vec{t}) \label{eq_tlwa_box}
\end{align}
} &
\parbox[b]{0.25\textwidth}{
\begin{align}
 \window^n @_T\, p(\vec{t}) &\mapsto @_T\, p(\vec{t}) \label{eq_tlwa_at}\\
 \window^n \Diamond p(\vec{t}) &\mapsto @_U\, p(\vec{t}) \label{eq_tlwa_diamond}
\end{align}
}
\end{tabular}
\end{linenomath}

\vspace{-0.5\baselineskip}

\noindent where \eqref{eq_tlwa_simple} refers to atoms with no surrounding
\elars operators and $U$ in \eqref{eq_tlwa_diamond} is a fresh time variable
unique for each replacement; arithmetic atoms are kept unchanged.  The resulting
program is denoted by $\nowindows{P}$ (``\emph{window-free}'').

\begin{example}
\label{ex_tlwa_nowin}
Let $P$ consist of the following rules:
\begin{linenomath}
\begin{align}
\window^3\Box\, p(X) &\to\exists Y.q(X,Y) \label{rul_ex_tlwa_exists}\\
@_T\, q(X,Y)\wedge U\,{=}\,T\,{+}\,1 & \to @_U\, p(Y) \label{rul_ex_tlwa_next}
\end{align}
\end{linenomath}
As in Example~\ref{ex_univTermOverestimate}, the skolem chase on $\rewrite{P}$ terminates if
the given input data encodes a valid timeline,
else it may not (indeed, $P\not\in\CL{LWA}$).
In $\nowindows{P}$, \eqref{rul_ex_tlwa_exists} is changed to
\begin{linenomath}
\begin{align}
@_N\, p(X) &\to\exists Y.@_N\,q(X,Y) 
\end{align}
\end{linenomath}
\end{example}

Intuitively, $N$ is in $\nowindows{P}$ the time at which rules are evaluated and
localises all simple atoms to it; windows are removed and their restrictions
relaxed: $\window^n \Box$ (``at all times in window up to now'') becomes $@_N$
(``now''); $\window^n @_T$ (``at $T$ if in window'') becomes $@_T$; and
$\window^n \Diamond$ (``at some time in window'') becomes $@_U$ (``at some
time'').  As this logically weakens rule bodies, $\nowindows{P}$ has more
logical consequences than $P$.  We obtain the following useful insight:

\begin{restatable}{theorem}{thmNowinTermination}\label{theo_nowin_termination}
For every $P\in\Lfull$ and data stream $D$,
if the skolem chase terminates on
$\rewrite{\nowindows{P}}$ and $\rewrite{D}$,
then it also terminates on $\rewrite{P}$ and $\rewrite{D}$.
\end{restatable}

To exploit Theorem~\ref{theo_nowin_termination}, we study the chase termination
over $\rewrite{\nowindows{P}}$ while restricting to actual timelines, which are
incorporated by partial grounding.

\begin{definition}\label{def_partGround}
The \emph{partial grounding}\/ $\ground_A(P)$ of a program $P$ for a set $A$ of
null-free facts over a set $\Predicates_A$ of predicates not occurring in rule
heads of $P$, is the set of all rules $(B{\setminus} B_A \,{\rightarrow}\,
\exists\vec{z}.H)\sigma$, where $B_A$ are the atoms in $B$ with predicate in
$\Predicates_A$, s.t.\ a rule $B \,{\rightarrow}\,\exists\vec{z}.H \,{\in}\, P$
and a homomorphism $\sigma$ between $B_A$ and $A$ exist, i.e., a sort- and
constant- preserving mapping $\sigma:\dom{B_A}\to \dom{A}$ s.t.\
$B_A\sigma\,{\subseteq}\,A$.
\end{definition}

As long as $A$ comprises all facts over $\Predicates_A$, $\ground_A(P)$ has the
same models as $P$ and the chase is also preserved. We use this to ground the
time sort in \elars:

\begin{definition}\label{def_tempground}
Given a program $P$, the \emph{temporal grounding}\/ of $\nowindows{P}$ for a
timeline $\timeline$, denoted $\tground_\timeline(P)$, is the partial grounding
$\ground_{a(\timeline,P)} (P')$ where
\begin{itemize}
\item $P'$ results from $\rewrite{\nowindows{P}}$ by adding, for each
$T\in\Variables_T$ in each rule body $B$, an atom
$T\leq T$ to $B$ and
\item $a(\timeline,P)$ is the set of all ground instances of arithmetic atoms in $P$ with values from $\timeline$
that are true over $\mathbb{N}$.
\end{itemize}
\end{definition}

\begin{example}\label{ex_tlwa_tempground}
For $\nowindows{P}$ from Example~\ref{ex_tlwa_nowin}
and timeline $\timeline=[0,1]$,
the temporal grounding is as follows (the deleted ground instances
of $B_A$
are shown in parentheses):
\begin{linenomath}
\begin{align*}
\freshPred{\window\Box\,p}(X,0,0) &\to\exists Y.\freshPred{\window\Box\,q}(X,Y,0,0) & (0\,{\leq}\,0)\\
\freshPred{\window\Box\,p}(X,0,1) &\to\exists Y.\freshPred{\window\Box\,q}(X,Y,0,1) & (1\,{\leq}\,1)\\
\freshPred{\window\Box\,q}(X,Y,0,0) & \to \freshPred{\window\Box\,p}(Y,0,1) & (1\,{=}\,0\,{+}\,1)
\end{align*}
\end{linenomath}
\end{example}

While universal termination on $\tground_\timeline(P)$,
which can be recognized in Example~\ref{ex_tlwa_tempground}
using e.g.\ MFA \cite{DBLP:journals/jair/GrauHKKMMW13},
ensures chase termination on $\rewrite{P}$ and $\rewrite{D}$ for all data streams $D$ on $\timeline$,
simpler, position-based notions like WA still fail.
We thus encode
time into predicate names:

\begin{definition}\label{def_noTime}
Let $P$ be an existential rules program with atoms of form
$\freshPred{\window\Box\,p}(\vec{s},0,t)$ only,
where $t$ is a time point.
Then $\notime{P}$ is obtained by replacing each $\freshPred{\window\Box\,p}(\vec{s},0,t)$
with $\freshPred{p}_t(\vec{s})$ for a fresh predicate $\freshPred{p}_t$ of proper signature.
\end{definition}

Let  $\tfground_\timeline(P):= \noTimeGroundNoWindow{P}{\timeline}$. The
following result shows that this is a good basis to check for acyclicity.

\begin{restatable}{theorem}{thmtlwa}\label{theo_tlwa} If
    $\tfground_\timeline(P)$ is weakly acyclic for $P\,{\in}\,\Lfull$ and
    timeline $\timeline$, then the skolem chase terminates on $\rewrite{P}$ and
    $\rewrite{D}$ for all data streams $D$ on $\timeline$.
\end{restatable}

\begin{example}
[cont'd] \label{ex_tlwa_term} As $\tfground_\timeline(P)$ is WA, by
Theorem~\ref{theo_tlwa} the skolem chase on $\rewrite{P}$ and $\rewrite{D}$
always terminates.
\end{example}

In view of Theorem~\ref{theo_tlwa}, we call $P \,{\in}\, \Lfull$
\emph{temporally weakly acyclic}\/ (TLWA) over $\timeline$ if
$\tfground_\timeline(P)$ is WA, and denote by $\CL{TLWA}(\timeline)$ the class
of all such programs $P$.  We then have:
\begin{restatable}{theorem}{thmrelclasses}\label{theo_wa_lwa_tlwa}
$\CL{LWA}\,{\subset}\,\CL{TLWA}(\timeline)$ holds for all $\timeline$ s.t.\
$|\timeline|\,{\geq}\,2$.  \end{restatable}

Regarding complexity, as $\rewrite{\timeline}$ is polynomial \emph{in the length
of $\timeline$}, it is exponential if $\timeline$ is encoded in binary.
However, a polynomial axiomatisation of time is feasible, following the idea to
encode numbers 0,1,\ldots,$m$ using sequences of $\lceil\log_2 m\rceil$ bits and
to define predicates on them, cf.\ \cite{D+:datalogcomp}, such that for the
resulting rewriting $\rewriteT{\cdot}$ instead of $\rewrite{\cdot}$,
Theorems~\ref{theo_elars-in-exrules} and \ref{theo_acyc_stripped} hold
analogously.

BCQ answering for $\CL{TLWA}(\timeline)$ is as for WA rules \doubleexp-complete
in general (on extensional streams, i.e, all $v(t)$, $t\,{\in}\,\timeline$, are
listed).  The \complclass{P}-complete data complexity for WA rules carries over
to $\CL{LWA}$ but  gets \doubleexp-hard for $\CL{TLWA}(\timeline)$, as hardest
WA programs with bounded predicate arities \cite{cali_query_2010} can be
emulated.

\section{Preliminary Evaluation and Conclusion}
\label{sec:evaluation}

We implemented an experimental prototype in Python, which is fed with the stream
pointwise. At each time point, it computes the \elars{} model with the stream
collected up to the last $\ell$ time points, using the rewriting in
Section~\ref{sec_elars-to-tgds} and the chase implementation of
GLog~\cite{glog}.

We considered two scenarios $S_A$ and $S_B$. The first, $S_A$, is a toy example
with conveyor belts and sensors that measure speed and temperature.  The program
contains 5 simple rules and the stream is parametrized by probability values
$p_1$, $p_2$, and $p_3$ that regulate the number of rule executions (higher
values lead to more reasoning). Scenario $S_B$ is much more complex than $S_A$.
We considered the dataset $\deep$ from the $\chasebench$ suite
\cite{bench_chase}, which is a stress test of chase engines. We
created a stream by copying all facts on each time point and rewrote the
original rules using \elars operators and different window sizes $n$. %
\ifthenelse{\boolean{TR}}{More details are available in
    Appendix~\ref{app:experiments}.}{More details are
available at~\cite{tr}.}

Table~\ref{tab:experimental-results} reports multiple metrics obtained using a laptop, viz.\ avg.\ runtime (\emph{Run}),  avg.\
peak use of RAM (in MB, \emph{Mem}), and avg.\ model size (\# facts,
\emph{Out}).
\begin{table}[t]
\caption{Preliminary experiments for scenario $S_A$ and $S_B$}
\label{tab:experimental-results}
\newcommand{\mycolspace}{~~~}
\newcommand{\mycolspacesmall}{\,}
\small
\renewcommand{\arraystretch}{1.05}
\centering\begin{tabular}{@{}c@{~~~~}c@{}}
\begin{tabular}{@{}l@{\!\!\!}c@{\mycolspacesmall}|@{\mycolspacesmall}c@{\mycolspacesmall}|@{\mycolspacesmall}c@{\mycolspacesmall}|@{\mycolspacesmall}r@{}}
  $S_A$: & ~~~~$p_1$/$p_2$/$p_3$ & \em Run & \em Mem & \multicolumn{1}{@{\mycolspacesmall}c@{}}{\em Out} \\  \cline{2-5}
       & $0.0$/$0.0$/$0.0$ & 13.12ms & 21.9 & 10.5k  \\
       & $0.3$/$0.3$/$0.5$ & 13.34ms & 22.7 & 10.7k  \\
       & $0.7$/$0.7$/$1.0$ & 13.67ms & 22.9 & 10.7k  \\
\end{tabular}
&
\begin{tabular}{@{}l@{\,}c@{\mycolspacesmall}|@{\mycolspacesmall}c@{\mycolspacesmall}|@{\mycolspacesmall}c@{\mycolspacesmall}|@{\mycolspacesmall}r@{}}
$S_B$: &  $n$ & \em Run & \em  Mem & \multicolumn{1}{@{}c@{\,}}{\em Out} \\  \cline{2-5}
       & 0 & 0.6s & 45.0 & 36k \\
       & 2 & 1.3s & 81.8 & 64k \\
       & 4 & 2.6s & 114.4 & 82k \\
\end{tabular}
\end{tabular}
\end{table}
\noindent Notably, a \elars{} model can be computed rather quickly, viz.\ in
$\approx$13ms with an hypothetical input like $S_A$. This suggests that our
approach can be used in scenarios that need fast response times.  For
``heavier'' scenarios like $S_B$, the runtime increases but still stays within
few seconds. Moreover, reasoning used at most 114MB of RAM; thus it may be done
on limited hardware, e.g., sensors or edge devices.

\medskip

\noindent\textbf{Conclusion.} Our work shows that combining existential rules
with LARS can give rise to a versatile stream reasoning formalism with
expressive features which is still decidable. A worthwhile future objective is
to develop more efficient algorithms to compute the models. Our translation to
existential rules is a good basis, but many optimisations are conceivable. On
the theoretical side, a study of further decidability paradigms, especially
related to guarded logics, is suggestive.  Finally, further extensions towards
non-monotonic reasoning or other issues, like window validity
\cite{DBLP:conf/kr/RoncaKGH18}, are challenging for existential rules, but would
be very useful for stream reasoning.

\vspace{0.5em}
\noindent{\emph{Acknowledgments.}} The authors would like to thank Mike Spadaru for
his work on earlier versions of the prototype used in this work.


\ifthenelse{\boolean{TR}}{\cleardoublepage
\appendix
\section{Proofs Section \ref{sec_elars-to-tgds}}

\thmLARSplusIntoExrules*
\begin{proof}
Let $\eta$ be the function that maps a stream $S$ to the deductive closure of $\rewrite{S}$ under the rules
\eqref{rul_top}--\eqref{rul_boxexpandat}.
We show that, for any $t\in\timeline$, we have $S,t\models q$ iff $\eta(S)\models_\timeline \rewrite{q,t}$,
where $q=\exists\vec{x}.Q$.
Analogous results will also be established for $P$ and $D$.

Therefore, let $S=(\timeline,v)$ be such that $S,t\models q$ as in Definition~\ref{def_elars_rule_sem}.
We show that $\eta(S)\models_\timeline \rewrite{q,t}$.
By $S,t\models q$, there is a $\timeline$-match $\sigma$ of $Q$ on $S$ and $t$.
Therefore, for every atom $\alpha$ in $Q$, $S,t\models \alpha\sigma$.
Our rewriting replaces atoms $\alpha=\window^n \Diamond p(\vec{t})$ by
$\alpha'=\window^n @_T\, p(\vec{t})$ for a fresh $T$. Clearly, whenever $S,t\models \alpha\sigma$,
there is a suitable $s\in\timeline$ such that $S,t\models \alpha'\sigma\{T\mapsto s\}$.
We can therefore construct a query without $\Diamond$ and a suitable $\timeline$-match
over $S$ and assume without loss of generality that $\Diamond$ does not occur in the query.

We claim that $\sigma_t :=\sigma\cup\{C\mapsto t\}$ is a match for
$\rewrite{q,t}=\exists\vec{x}.\bigwedge_{\alpha\in Q} \rewrite{\alpha}\wedge C\,{\leq}\,t\wedge t\,{\leq}\,C$
over $\eta(S)$.
Clearly, $\eta(S),t\models_\timeline (C\,{\leq}\,t)\sigma_t$ and likewise for $t\,{\leq}\,C$.
For the other atoms $\rewrite{\alpha}$ with $\alpha\in Q$, we can show $\eta(S),t\models\alpha\sigma_t$ by considering each
possible form of atom:
\begin{itemize}
\item For $\alpha=p(\vec{t})$, we obtain $p(\vec{t})\sigma\in v(t)$.
      Since $\rewrite{\alpha}=\freshPred{\window\Box\,p}(\vec{t},0,C)$
      and $\rewrite{\alpha}\sigma_t=\freshPred{\window\Box\,p}(\vec{t}\sigma,0,t)$,
      we get $\eta(S)\models_\timeline\rewrite{\alpha}\sigma_t$ as required.
\item For $\alpha=@_{t'}\, p(\vec{t})$, we obtain $p(\vec{t})\sigma\in v(t')$,
  and the claim follows with a similar argument as in the previous case.
\item For $\alpha=\window^n \Box p(\vec{t})$, we obtain $p(\vec{t})\sigma\in v(t')$ for all
$t'\in\timeline$ with $t-n\leq t'\leq t$;
hence, we get $\eta(S)\models_\timeline\freshPred{\window\Box\,p}(\vec{t}\sigma,0,t')$ for every such $t'$ by a similar argument as before.
We have $\rewrite{\alpha}=\freshPred{\window\Box\,p}(\vec{t},n,C)$.
Since $\eta(S)$ satisfies rules \eqref{rul_boxinit}--\eqref{rul_boxexpand}, we find that
$\eta(S)\models_\timeline \freshPred{\window\Box\,p}(\vec{t},n,C)\sigma_t$: we can apply
rule \eqref{rul_boxexpand} on true atoms of the form $\freshPred{\window\Box\,p}(\vec{t}\sigma,0,t')$ to
infer windows of increasing sizes up until $n$; if $t-n<0$, then rule \eqref{rul_boxinit}
is used to start with a maximal window at time $0$, which can be reduced in size by rule \eqref{rul_boxshift},
before we again apply \eqref{rul_boxexpand} to infer the required $\freshPred{\window\Box\,p}(\vec{t}\sigma,n,t)$.
\item For $\alpha=\window^n @_{t'}\, p(\vec{t})$, we obtain $p(\vec{t})\sigma\in v(t')$ and
 $t-n\leq t'\leq t$.
 Using a similar argument as before, we can use rule \eqref{rul_boxinitat} to derive facts
 $\freshPred{\window @\, p}(\vec{t},0,t',t')$, and rule \eqref{rul_boxexpandat} to modify the window
 size and position to obtain $\freshPred{\window @\, p}(\vec{t},n,t',t)$.
\item The cases of atoms $\alpha$ that use $\top$ instead of $p(\vec{t})$ are shown in the same way,
 with the only difference that facts of the form $\freshPred{\window\Box\,p}(\vec{t}\sigma,0,s)$ are now
 replaced by facts of the form $\freshPred{\window\Box\,\top}(0,s)$, which are provided by rule \eqref{rul_top}.
\item For arithmetic atoms $\alpha$, the rewriting does not change the atom, and the claim is immediate.
\end{itemize}
This completes the argument that $\eta(S)\models_\timeline \rewrite{q,t}$.

Conversely, assume that there is a model $S'\models_\timeline \rewrite{q,t}$ that satisfies
\eqref{rul_top}--\eqref{rul_boxexpandat}, and such that $S'=\eta(S)$ for a suitable $S$.
We show that $S,t\models q$. The argument proceeds as before, but now using that
$\eta(S)$ contains only facts of form $\freshPred{\window @\, p}(\vec{t},n,s,t')$ or
$\freshPred{\window\Box\,p}(\vec{t},n',t')$ with $n'>0$ that are needed to satisfy some rule
\eqref{rul_top}--\eqref{rul_boxexpandat}.

We now also find that, for any rule $r\in P$, it holds that $S\models r$ iff $\eta(S)\models_T\rewrite{r}$,
where $\rewrite{r}$ denotes the result of rewriting a single rule $r$ as described before.
This is an easy consequence from the previous statement for queries, since $\timeline$-matches for
rule heads and bodies behave like query matches.
Note that $\body{\rewrite{r}}$ may contain not only the atoms in $\rewrite{\body{r}}$
but also an additional atom $\freshPred{\window \Box\,\top}(0,C)$.
However, the previous argument for queries still applies, since we can assume
w.l.o.g.\ that $\body{r}$ contains the atom $\top$, in which case $\body{\rewrite{r}}=\rewrite{\body{r}}$
does again hold.
Finally, we also note that $S\models D$ iff $\eta(S)\models_T\rewrite{D}$ for data streams $D$.

These correspondences already show that, whenever there is a stream $S$ with
$S\models P$ and $S\models D$ but $S,t\not\models q$, we find that
$\eta(S)\models_\timeline  \rewrite{P}$ and $\eta(S)\models_\timeline
\rewrite{D}$ but $\eta(S)\not\models_\timeline \rewrite{q,t}$.
For the converse direction, we note that, for any model $S'$ with $S'\models_\timeline \rewrite{P}$
and $S'\models_\timeline \rewrite{D}$, there is a model of the form
$\eta(S)\subseteq S'$ for some stream $S$ for which
$\eta(S)\models_\timeline\rewrite{P}$ and $\eta(S)\models_\timeline\rewrite{D}$, i.e., we can restrict attention to
models of the form $\eta(S)$, which provide the semantic correspondences shown above.
Indeed, a suitable $\eta(S)$ can be obtained by removing from $S'$ all facts of the form
$\freshPred{\window @\, p}(\vec{t},n,s,t')$ or $\freshPred{\window\Box\,p}(\vec{t},n',t')$ with $n'>0$,
and deductively closing the result under the rules \eqref{rul_top}--\eqref{rul_boxexpandat}.
\end{proof}

\section{Proofs Section~\ref{sec_tlwa}}

First, we provide a more description of the (skolem) chase and of the standard
notion of weak acyclicity.

The \emph{chase} is a versatile class of reasoning algorithms for existential
rules \cite{bench_chase}, which is based on ``applying'' rules iteratively until
saturation (or, possibly, forever).  We present a variant of the \emph{skolem
chase} \cite{Marnette09:superWA}, using nulls instead of skolem terms (this
version is sometimes called the \emph{semi-oblivious chase}), and extended to
the time sort.

Let $r$ be an existential rule of the form:
\begin{linenomath} \begin{align} r = \forall \vec{x}, \vec{y}.\ B[\vec{x},
\vec{y}] \to \exists \vec{z}.\ H[\vec{y}, \vec{z}]\label{eq_rule} \end{align}
\end{linenomath} where $B$ and $H$ are conjunctions of
normal atoms, and $\vec{x}, \vec{y}, \vec{z}$ are mutually disjoint lists of
variables.  $B$ is the \emph{body} (denoted $\body{r}$), $H$ the \emph{head}
(denoted $\head{r}$), and $\vec{y}$ the \emph{frontier} of $r$.  Notice that
below we may treat conjunctions of atoms as sets, and we omit universal
quantifiers in rules.

Moreover, let $A$ a set of facts.  A \timeline-match $\sigma$ for $\body{r}$
(defined on $\vec{x}$ and $\vec{y}$) is extended to a term mapping $\sigma^+$ by
setting, for each $Z\in\vec{z}$, $v\sigma^+=n^{r,Z}_{\vec{y}\sigma}$, which is a
fixed named null specific to $r$, $Z$, and $\vec{y}\sigma$.  The \timeline-match
$\sigma$ is \emph{active} for $A$ if $\head{r}\sigma^+\not\subseteq A$.

The \emph{skolem chase sequence} $F_0,F_1,\ldots$ over a program $P$ and a set
of null-free facts $A$ is specified as follows: (1) $F_0=A$ and (2) $F_{i+1}$ is
obtained from $F_i$ by adding $\head{r}\sigma^+$ for every rule $r\in P$ and
active \timeline-match $\sigma$ over $\body{r}$ and $F_i$.  The result of the
skolem chase is $\bigcup_{i\geq 0}F_0$ in this case.  The chase
\emph{terminates} if $F_{i+1}=F_i$ for some $i\geq 0$.  As usual, the
\emph{(skolem) chase} over $R$ and $A$ refers to this computation process or to
its result, depending on context.  Our definitions also apply to single-sorted
existential rules without time.

For some (finite) $I$, a chase procedure might not terminate and determining
this is undecidable in the most general case~\cite{chase_nontermination}.
Fortunately, many decidable conditions that guarantee chase termination were
proposed (\citeauthor{DBLP:journals/jair/GrauHKKMMW13}
\shortcite{DBLP:journals/jair/GrauHKKMMW13} give an overview and comparison).
Among them, \emph{weak acyclicity}\/ can be seen as a simple representative of
these approaches~\cite{fagin_data_2005}. Intuitively, the idea is to construct a
graph that we can use to track how variables ``propagate'' across the rules.
If such propagations do not generate any cycle that involves existentially
quantified variables, then we are sure the chase will always terminate.  We
describe the procedure more formally below.

\begin{definition}\label{def_wa}
For
a program $P$, we define a directed graph $G$
whose nodes are \emph{predicate positions}\,
$\tuple{p,i}$, where $p\,{\in}\,\Predicates$ and $1\,{\leq}\,i\,{\leq}\,\arity{p}$.
For a variable $X$ and set $A$ of atoms, let
$\funcStyle{pos}(X,A)\defeq\{\tuple{p,i}\mid
p(\vec{t})\,{\in}\,A\text{ and }t_i\,{=}\,X\}$
be the set of all positions where $X$ occurs in $A$.
For every rule $r$ as in \eqref{eq_rule},
frontier variable $Y\in\vec{y}$,
position $\pi\in\funcStyle{pos}(Y,\body{r})$, and
existential variable $Z\in\vec{z}$, we add two kinds of edges to $G$:
\begin{itemize}
\item a \emph{normal edge} $\pi\to\pi'$ for all $\pi'\in\funcStyle{pos}(Y,\head{r})$
\item a \emph{special edge} $\pi\stackrel{*}{\to}\pi'$ for all $\pi'\in\funcStyle{pos}(Z,\head{r})$
\end{itemize}
Then $P$ is \emph{weakly acyclic} (WA) if $G$ does not have a cycle through a special edge.
\end{definition}

Recall that we denote the class of all weakly acyclic programs with
$\mathbf{WA}$. We are now ready to discuss Theorem 2.

\thmAcycStripped*

\begin{proof}
1) We begin with the first claim, which refers to chase termination.
Let $C$ be the set of all possible ground instances of
arithmetic atoms, including ``nonsensical'' ones like, e.g., 0=1+0, in $\rewrite{P}$ using values from $\timeline$,
and let $C_\timeline$ be the analogous set of all ground instances of rewritten
arithmetic atoms in $\rewriteT{P}$ (where numbers are encoded in binary as explained before).

Let $F$ be an arbitrary set of input facts for $\rewrite{P}$ such that $C\subseteq F$ and
$F$ does not contain facts for predicates of the form $\freshPred{\window @\, p}$.
Then the skolem chase on $\rewrite{P}\cup F$ contains a fact $\freshPred{\window\Box\,p}(\vec{t},0,t)$
iff it contains every fact of the form $\freshPred{\window\Box\,p}(\vec{t},0,s)$ for $s\in\timeline$.
This follows from rules \eqref{rul_boxshift} and \eqref{rul_boxexpand} using the atoms of $C$.
Similarly, facts of the form $\freshPred{\window @\, p}(\vec{X},N,T,C)$ hold at all times and for
all window sizes if $\freshPred{\window\Box\,p}(\vec{t},0,t)$ is true for any $t\in\timeline$.
In other words, the skolem chase for $\rewrite{P}\cup F$ effectively merges deductions for all time points.

The skolem chase on $\rewrite{P}\cup F$ therefore corresponds to the skolem chase
on $\stripped{P}\cup F'$, where $F'=\{p(\vec{t})\mid \freshPred{\window\Box\,p}(\vec{t},0,t)\in F\}$.
Indeed, arithmetic atoms are always true on $F$ since $C\subseteq F$ and can therefore be ignored,
and all temporal operators can be omitted when all time points are merged.
In particular, the skolem chase on $\rewrite{P}\cup F$ terminates iff the skolem chase on
$\stripped{P}\cup F'$ terminates.
An analogous result holds for $\rewriteT{P}$ with inputs that contain $C_\timeline$.

Now to finish the proof of the first claim, consider $\stripped{P}\in\mathbf{CT}$ iff $\rewrite{P}\in\mathbf{CT}$.
First assume that there is a set of facts $F'$
such that $\stripped{P}\cup F'$ does not terminate. Every $F'$ is of the form $F'=\{p(\vec{t})\mid \freshPred{\window\Box\,p}(\vec{t},0,t)\in F\}$ for some $F$ with $C\subseteq F$ that contains no facts for predicates $\freshPred{\window @\, p}$.
Hence we find that $\rewrite{P}\cup F$ has no terminating skolem chase.

Conversely, assume that the skolem chase does not terminate on $\rewrite{P}\cup G$ for some set of input facts $G$
that may not satisfy the previous conditions on $F$.
We extend $G$ by adding, for every fact $\alpha_@=\freshPred{\window @\, p}(\vec{t},n,s,t)$ a new fact
$\alpha_\Box=\freshPred{\window\Box\,p}(\vec{t},0,t)$. This addition preserves non-termination of the
chase, as every addition of input facts does for the skolem chase.
As argued above, $\alpha_@$ follows from $\alpha_\Box$
using rules \eqref{rul_boxinit}--\eqref{rul_boxexpandat}, hence we can delete $\alpha_@$ from $G$ while
preserving non-termination. This leads to a non-terminating set $G$ without predicates $\freshPred{\window @\, p}$.
To satisfy the other condition on $F$, we can simply add $C$ to $G$, which again preserves non-termination.
The skolem chase on the resulting set $G$ then again corresponds to a skolem chase on $\stripped{P}$,
which establishes non-termination.
The case for $\stripped{P}\in\mathbf{CT}$ iff $\rewriteT{P}\in\mathbf{CT}$ is analogous.

2) For the second claim, we
first address $\stripped{P}\in\mathbf{WA}$ iff $\rewrite{P}\in\mathbf{WA}$.
Consider the graphs $G_r$ and $G_s$ as in Definition~\ref{def_wa} for $\rewrite{P}$ and $\stripped{P}$,
respectively.
The forward direction can be shown by establishing the following:
(a) for every normal edge $\tuple{p,i}\to\tuple{q,j}$ in $G_s$, there is a path
$\tuple{\freshPred{\window\Box\,p},i}\to\cdots\to\tuple{\freshPred{\window\Box\,q},j}$ in $G_r$; and
(b) for every special edge $\tuple{p,i}\stackrel{*}{\to}\tuple{q,j}$ in $G_s$, there is a path
$\tuple{\freshPred{\window\Box\,p},i}\to\cdots\stackrel{*}{\to}\tuple{\freshPred{\window\Box\,q},j}$ in $G_r$.
Together, (a) and (b) imply that every cycle in $G_s$ that involves a special edge
also leads to such a cycle in $G_r$, showing the first part of the claim.

There are two kinds of rules in $\rewrite{P}$: rewritten versions of rules in $P$
and auxiliary rules to axiomatise temporal operators.
To show (a) and (b), note that the heads of rewritten rules in $\rewrite{P}$ only contain
atoms of the form $\freshPred{\window\Box\,p}(\vec{t},N,T)$, and that normal and special edges in rewritten
rules are analogous to those in $G_s$.
However, rewritten rules may also contain body predicates $\freshPred{\window @\, p}$.
The claim follows by noting that, for every predicate position $\tuple{p,i}$,
$G_r$ contains a normal edge
$\tuple{\freshPred{\window\Box\,p},i}\to\tuple{\freshPred{\window @\, p},i}$
due to rule \eqref{rul_boxinitat}.

For the converse direction, we can use a similar correpondence between paths in $G_r$ and paths in $G_s$.
However, we additionally need to observe that, for any predicate $p$ of arity $a$, the additional argument positions
$\tuple{\freshPred{\window\Box\,p},a+1}$, $\tuple{\freshPred{\window @\, p},a+1}$, and
$\tuple{\freshPred{\window @\, p},a+2}$ do not occur in any cycle that involves a special edge.
This is an easy consequence of the fact that those positions represent arguments of the time sort.
Therefore, we find that every cycle in $G_r$ that has a special edge corresponds to such a cycle in $G_s$.
The argument for $\stripped{P}\in\mathbf{WA}$ iff
$\rewriteT{P}\in\mathbf{WA}$ is again similar.
\end{proof}

\paragraph{Remark.} For the following Theorem~\ref{theo_nowin_termination} and the definitions of
$\ground_\timeline(\nowindows{P})$ (Definition~\ref{def_tempground}) and
$\CL{TLWA}(\timeline)$, we assume that
$\rewrite{\nowindows{P}}$ does not contain any of the auxiliary rules
\eqref{rul_boxexpand}--\eqref{rul_boxexpandat}. Indeed, these rules are not relevant for
chase termination in an existential rule set where atoms of the form
$\freshPred{\window\Box\,p}(\vec{t},n,c)$ only occur with $n=0$ and predicates
$\freshPred{\window @\, p}$ do not occur at all.

\thmNowinTermination*
\begin{proof}
The claim follows from our previous observation that the rules in $\nowindows{P}$ have more consequences than those in $P$.
Indeed, the skolem chase is monotonic with respect to the amount of entailments, hence the result of a skolem chase
on $\rewrite{\nowindows{P}}$ is a superset of the result of a skolem chase on $\rewrite{P}$.
\end{proof}

Consider the partial grounding introduced in Definition~\ref{def_partGround}. We
stated that as long as $A$ comprises all facts over  $\Predicates_A$,
$\ground_A(P)$ has the same models as $P$ and the chase is also preserved. This
statement can be restated as follows.

\begin{restatable}{lemma}{lemmaPartGroundWorks}\label{lemma_partGroundWorks}
Consider a program $P$ and a set  $A$ of null-free facts over $\Predicates_A$ as in Definition~\ref{def_partGround}.
If $B$ is a fact set such that $A=\{p(\vec{t})\in B\mid p\in\Predicates_A\}$, then
the skolem chase on $P$ and $B$ is the same as the skolem chase on $\ground_A(P)$ and $B$.
\end{restatable}

\begin{proof} The claim follows because every $\timeline$-match on a rule of $P$
    must, by definition, instantiate all body atoms for a predicate in
    $\Predicates_A$ to a fact in $A$. Since the program $\ground_A(P)$ contains
    a rule for every possible choice of fact from $A$, an analogous
    $\timeline$-match is applicable on $\ground_A(P)$, and the chases are based
on applications of the same ground rule instances.  \end{proof}

We now consider that universal termination on
$\tground_\timeline(P)$ ensures chase termination on $\rewrite{P}$ and
$\rewrite{D}$.

\begin{lemma}\label{theo_tempground_termination} If the skolem chase
    universally terminates on $\tground_\timeline(P)$ 
    for
    $P\,{\in}\,\Lfull$ and timeline $\timeline$, then it terminates on
    $\rewrite{P}$ and $\rewrite{D}$ for all data streams $D$ on $\timeline$.
\end{lemma}

\begin{proof} Combine Theorem~\ref{theo_nowin_termination} and
Lemma~\ref{lemma_partGroundWorks}. \end{proof}

\noindent Finally, the following lemma is easy to show.

\begin{lemma}\label{lemma_noTime_correct}
For $P$ as in Defn.~\ref{def_noTime} and a fact set $A$ over predicates in $P$,
the results of the skolem chase on $P$ and $A$ resp.\ on
$\notime{P}$ and $\notime{A}$ are in a bijective correspondence.
\end{lemma}

\noindent We are now ready to discuss Theorem~\ref{theo_tlwa}.

\thmtlwa*
\begin{proof}
Suppose that the set $\noTimeGroundNoWindow{P}{\timeline}$ of
existential rules is weakly acyclic. Then the skolem chase universally
terminates for it. Then, Lemma~\ref{lemma_noTime_correct} ensures that the skolem
chase universally terminates for $\tfground_\timeline(P)$.
From
Lemma~\ref{theo_tempground_termination}, it then follows that the skolem chase
terminates on $\rewrite{\nocarry{P}}$ and $\rewrite{D}$ for every data
stream $D$ over $\timeline$.
\end{proof}

\thmrelclasses*

\begin{proof} From $\noTimeGroundNoWindow{P}{\timeline}$ we can obtain $\stripped{P}$
by a surjective
renaming of predicates $\freshPred{p}_t\mapsto p$,
and likewise the graph for WA by collapsing vertices, which preserves
cycles.
By Example~\ref{ex_tlwa_term},
$\CL{TLWA}(\timeline)\,{\not\subseteq}\,\CL{LWA}$ for a timeline of size~2.
\end{proof}

\section{Notes on complexity}

Below, we provide a more detailed description of the complexity of the proposed
procedures to ensure decidability and to perform BCQ answering with LARS+
programs.

\subsection{Complexity of deciding \CL{LWA} and \CL{TLWA}}

\begin{theorem}\label{theo_lwa_complexity}
Given a \elars program $P$, deciding whether $P\,{\in}\,\CL{LWA}$
is \nlogspace-complete.
\end{theorem}

\begin{proof}[Proof (Sketch)]
To decide whether $P\,{\in}\,\CL{LWA}$ holds, we need to check whether the
dependency graph $G$ for $\stripped{P}$ has no cycle containing a
special edge. Each node of $G$ (consisting of a pair $\langle p,i \rangle$) can be stored in logarithmic space, and
deciding whether between two given nodes a normal resp.\ special edge exists is feasible in
logarithmic space. Therefore, a cycle that contains a special edge can be
non-deterministically guessed and checked stepwise in logarithmic
space. Since \nlogspace{} = co-\nlogspace,
this establishes \nlogspace{} membership of the problem.

The \nlogspace-hardness is inherited from the \nlogspace-hardness of WA
checking of existential rules, which can be proved by a simple
reduction from the graph reachability problem. Indeed, given a
directed graph $G= (V,E)$ and a starting/end node $s$/$t$ from $V$, we introduce a unary predicate $p_v$ for each
$v\in V$, and a binary predicate $q$. We then set up rules
$p_v(X) \to p_{v'}(X)$ for all edges $v \rightarrow v'$ in $E$ and for the
start resp.\ terminal node the rules
 $p_t(X) \rightarrow \exists Y. q(X,Y)$  and $q(X,Y) \rightarrow
p_s(Y)$, respectively. Then $G$ has a cycle with a special edge iff
there is a path from $s$ to $t$.
\end{proof}

Turning to temporal acyclicity, we first note that already computing the temporal grounding of rules is
intractable.

\begin{proposition}\label{prop_tempground_complexity}
Deciding, given a rule $r$ and a set $A$
of null-free facts over predicates
not occurring in $\body{r}$, whether $\mathsf{ground}_A(\{r\})$ is non-empty is
\np-complete, and \np-hard even if $\mathsf{ground}_A(\{r\})$
is a temporal grounding as in Definition~\ref{def_tempground} over any timeline with at least two elements.
\end{proposition}

\begin{proof}
Partial grounding requires to find a homomorphism from the conjunction $\body{r}_A$
($B_A$ in Definition~\ref{def_tempground}) to $A$, which is \np-complete in general.
For the special case of temporal grounding, we can show that it is still \np-complete
to find a homomorphism from a conjunction of arithmetic atoms to the fact encoding of a timeline.
For example, to encode three-colourability of a graph, every vertex $v$ is assigned three variables
$V_r$, $V_g$, $V_b$. We use atoms to express $0\leq V_x\leq 1$ and
$V_x+V_y\leq 1$, for all $x,y\in\{r,g,b\}$ with $x\neq y$. Every edge $v\to w$ of the graph
is encoded as $V_x+W_y\leq 1$ for all $x\in\{r,g,b\}$.
\end{proof}


Deciding temporal acyclicity has presumably higher complexity.

\begin{theorem}\label{theo_tlwa_complexity}
Given a \elars program $P$ and a timeline $\timeline$,
deciding whether $P\,{\in}\,\CL{TLWA}(\timeline)$ is
\pspace- complete.
\end{theorem}

\begin{proof}[Proof sketch]
Membership in \pspace{} follows as the required check corresponds to a reachability check on a
graph of exponential size whose edge relation can be validated in \pspace{}.
For hardness, we can simulate acceptance of a polynomial space bounded Turing machine by
using positions of the form $\tuple{p_t,1}$ to represent configurations of the TM, where the binary
encoding of $t$ represents the (polynomial size) tape contents and $p$
one of the (polynomially many) combinations of state and head positions.
TM transitions are encoded by normal edges that are based on rules that use (polynomially long)
conjunctions of arithmetic atoms to extract bits from time points, and to relate bits of consecutive
configurations.
The key for
bit extraction is to define variables $X_n$ that must be one of $\{0,2^n\}$ for any position $n$
on the tape. For $n=0$, we can encode $0\leq X_n\leq 1$. For $n+1$, we
can (recursively) define an auxiliary variable $X'_n$ with values in $\{0,2^n\}$ and require
$X_{n+1}=X'_n+X'_n$.
Note that one more auxiliary variable is used on each level, so the encoding of
the polynomially many $X_n$ is still polynomial. We can then
render a given timepoint as a sum
$\sum_{i=0}^\ell X_i$
to obtain a binary decoding. It is then easy to define rules for each TM transition.
To reduce TM acceptance to weak acyclicity, it remains to create special edges from each accepting configuration to
each starting configuration.
\end{proof}

\subsection{$\bcqname$ Answering}

\begin{table}[t]
 \newcommand{\psh}{\textsc{PS}}
\centering
\begin{tabular}{lcc}
		 &  \multicolumn{2}{c}{$\bcqname_\Lany$} \\
        & \em Data complexity & \em Combined complexity \\
        \cline{2-3}
        \CL{LWA} & \ptimeclass-c & \doubleexp-c  \\
		$\CL{TLWA}(\timeline)$ &\doubleexp-c & \doubleexp-c
\end{tabular}
    \caption{Complexity of \elars classes (c=complete).}\label{tab:1}
\end{table}

The decision problem that corresponds to BCQ answering, defined in Definition~\ref{def:bcq}, is stated below.
\smallskip

\noindent
\centerline{%
    \framebox{\begin{minipage}{0.92\linewidth}
    \textbf{Problem \elars BCQ Answering} ($\bcqname$) \\[2pt]
    \emph{Input:} \elars program $P$, data stream $D=(\timeline,
    v_D)$,
\phantom{\quad}   time point $t\in\timeline$, and \elars BCQ $q$.\\
\emph{Question:} Does $P,D,t\models q$ hold?
\end{minipage}}}\smallskip

Below, we will add a subscript to \bcqname{} to indicate which class $P$
belongs. For instance, $\bcqname_\Lany$ means that $P$ is supposed to be in $\Lany$.

\begin{restatable}{theorem}{thmbcq}
The complexity of the problem $\bcqclass{\Lany}$ is for $\Lany \in \{\CL{LWA}, \CL{TLWA}(\timeline)\}$ as reported in Table~\ref{tab:1}.
\end{restatable}

This result is obtained from several lemmas presented in the following
subsections. We make some useful observations about instances of $\bcqname$:

\begin{proposition}
\label{prop:bcq-aux-1}
Given an instance $P,D,t,q$ of $\bcqname$, let $P'$ result from
$P$ by replacing each window size $n$ occurring in $P$ such that $n \geq
|\timeline|$ with $|\timeline|-1$. Then  $P,D,T,q$ is a yes-instance of $\bcqname$
iff  $P',D,t,q$ is a  yes-instance of $\bcqname$.
\end{proposition}

As $P'$ in Proposition~\ref{prop:bcq-aux-1} is easily constructed from $P$ and $\timeline$, we thus may assume
assume without loss of generality that the largest
window size $m$ occurring in $P$ of an $\bcqname$ instance satisfies $m \leq |\timeline|-1$.

\begin{proposition}
\label{prop:bcq-aux-2}
Given an instance $P,D,t,q$ of $\bcqname$, let $P'$ result from
$P$ by adding the rule  $r_q := q\wedge @_N\,\pred{time}_q\wedge \window^0 @_N\,\top \rightarrow @_0\,\pred{yes}$, where
$\pred{time}_q$ and $\pred{yes}$ are fresh nullary predicates. Then  $P,D,t,q$  is a
yes-instance of $\bcqname$ iff  $P',D',t,\pred{yes}$  is a
yes-instance of $\bcqname$, where $D'= (\timeline,v'_d \})$ with
$v'_d(t)=v_d(t)\cup\{ \pred{time}_q \}$ and $v'_d(t')=v_d(t')$ for $t'\neq t$.
\end{proposition}

That is, we can compile the query $q$ into the program $P$ such that the
new query is a simple propositional atom, with little effort. The membership
of $P'$ in any of the classes $\Lany \in \{\CL{LWA}, \CL{TLWA}(\timeline)\}$ that we consider
coincides with the membership of $P$ in the respective class  $\bcqclass{\Lany}$.

In the sequel, we exploit Propositions~\ref{prop:bcq-aux-1} and \ref{prop:bcq-aux-2}
and restrict without loss
of generality  our attention for deriving upper boundes to $\bcqname$ instances
where the largest window size is at most $|\timeline|-1$ and the query is a simple
nullary (i.e., propositional) atom. Furthermore, we assume without
loss of generality that for any atom $@_T\,b$ occurring in rules of
$P$, the term $T$ is a time variable (if not, we can replace the atom
by $@_T' b$ and add $T'=T$ in the rule body,%
\footnote{We use $t_1 = t_2$ for time terms $t_1$ and $t_2$ as a
  shorthand for $t_1\,{\leq}\,t_2 \land t_2\,{\leq}\, t_1$.}
where $T'$ is a fresh time variable).
We call instances of $\bcqname$ which satisfy these properties {\em trimmed}.

\subsubsection{BCQ Answering with \CL{LWA}}
\label{app:lwa}

\begin{lemma} \label{lem:complexity-lwa} Problem $\bcqclass{\CL{LWA}}$ is
(i) in \ptimeclass{} under data complexity and (ii) in \doubleexp{}
under combined complexity.
\end{lemma}

\begin{proof}[Proof (Sketch)]
 Without loss of generality, the instance is trimmed.  \emph{(i)} For $P \in \CL{LWA}$, the number of
abstract constants and nulls generated in a chase of
$P' = \ground_\timeline(\nowindows{P})$ can bounded similarly as in Lemma~\ref{lem:tlwa-nulls-bound} below,
but with a smaller value of $s= n_p\,{\times}\,a$, since the time
arguments can be ignored and each null value must be generated at
stage $s$. Thus, in the naive bound
$b=(k\,{\times}\,|P|\,{\times}\,|\timeline|^\ell\,{\times}\,\ell\,{\times}\,n_c)^{\ell^s}$
the exponents $\ell$ and $\ell^s$ are constant; furthermore also the
bound for ${n'_r}^{\ell^s}$ in (\ref{eq:tlwa-n_r}) is then polynomial,
and so overall we obtain that the number of abstract constants and nulls
generated is polynomially bounded.

Along a similar argumentation as in Lemma~\ref{lem:complexity-tlwa-upper},
it can be shown that only a polynomial number
of atoms will be generated, and each step can be done in polynomial
time; hence we obtain overall a polynomial time algorithm for BCQ
answering.

In case (ii), the result follows from Theorem~\ref{theo_wa_lwa_tlwa} and
Lemma~\ref{lem:complexity-tlwa-upper}.
\end{proof}

\begin{lemma} \label{lem:complexity-lwa-lower}
Problem $\bcqclass{\CL{LWA}}$ is (i) \ptimeclass-hard
under data complexity and  (ii) \doubleexp-hard under combined complexity.
\end{lemma}

\begin{proof}
In case (i), the result follows from Theorem~\ref{theo_wa_lwa_tlwa} and
from the fact that the complexity of BCQ answering from
datalog programs is \ptimeclass-complete under data complexity.

In case (ii), the result is trivially inherited from the complexity of BCQ answering from
WA existential rules, which is \doubleexpcomplete{} under combined complexity
\cite{cali_query_2010}, taking into account that the problem can be
simply reduced to an empty data stream $D$ with timeline $\timeline = [0,0]$.
\end{proof}

\subsubsection{BCQ Answering with $\CL{TLWA}(\timeline)$}
\label{app:tlwa}

\paragraph{}{The following lemma is the key for obtaining upper bounds for BCQ
answering with programs in $\CL{TLWA}(\timeline)$.}


\begin{lemma}
\label{lem:tlwa-nulls-bound}
Given a trimmed instance of $\bcqclass{\CL{TLWA}(\timeline)}$,
in the skolem chase of $P' = \ground_\timeline(\nowindows{P})$ over
$\rewrite{D}\cup \rewrite{T}$, at most double exponentially
many abstract constants and nulls in the size of $P$ and $D$ occur. This number can
be (naively) bounded by
\begin{equation}
\label{eq:2exp-upper-bound}
b = (k\,{\times}\,|P|\,{\times}\,|\timeline|^\ell\,{\times}\,\ell\,{\times}\,n_c)^{\ell^s}
\end{equation}
where
\begin{itemize}
 \item  $\ell$ is the maximal rule length in $P$;
 \item $n_c$ is the number of abstract constants in $P$ and $D$;
 \item  $s= |\timeline|\,{\times}\,n_p\,{\times}\,a$, where
$n_p$ and $a$ are the number of predicates and the maximal predicate arity in $P$, respectively.
\end{itemize}
\end{lemma}

\begin{proof}
We consider the application of the existential rules in $P'$ for null value
generation as in the skolem chase. Starting from constants, the first null values are
generated by applying rules where all frontier variables are
substituted with constants; then the null values generated can take
part in generating further null values etc. Each null value is
generated at a {\em stage}\/ $i\geq 0$, which corresponds to the
step of the skolem chase sequence where it first appears.

As $P$ is in $\CL{TLWA}(\timeline)$, the program $P'$ is weakly
acyclic. We claim that each null value will be generated up to at most
stage $s= |\timeline|\,{\times}\,n_p\,{\times}\,a$.

Recall that, as remarked earlier, the rules \eqref{rul_boxexpand}--\eqref{rul_boxexpandat}
are not included in the rewriting for $P'$.

The program $P'$ therefore is of the form in Lemma~\ref{lemma_noTime_correct}, and we have
in $\notime{P'}$ predicates $\freshPred{p}_t$ where $p$ is a simple
predicate from $P$ and $t$ is a time point in $\timeline$.

In the dependency graph $G$ for $\notime{P'}$, the number of nodes is
thus bounded by $|\timeline|\,{\times}\,n_p\,{\times}\,a$ (note that
the nullary predicates $\freshPred{\top}_t$ do not create nodes in $G$).

Each null value $\omega$ of stage $i$ can flow by rule applications
along different predicate argument positions,
reflected by normal edges in $G$, until it is used in creating a null
value $\omega'$ of stage $i+1$. As  $P'$ is weakly acyclic, also
$\notime{P'}$ is weakly acyclic and so the newly created null value $\omega'$
can not flow to any predicate argument positions at which $\omega$ was present.
Hence, the number of stages for null value generation is bounded by
$|\timeline|\,{\times}\,n_p\,{\times}\,a$.

We will now argue that the bound $b$ claimed in
(\ref{eq:2exp-upper-bound}), i.e. $b = (k\,{\times}\,|P|\,{\times}\,|\timeline|^\ell\,{\times}\,\ell\,{\times}\,n_c)^{\ell^s}$
holds and that it is double exponential in the size
of $P$ and $D$.

Now let $n_r = |P'|$ denote the number of rules in $|P'|$.
Let $v_i$ denote the number of abstract constants and null values that we have
up to the $i^{th}$ stage of null value generation. If $i=0$, then we have
not have applied any rules to generate null values and we thus have $v_0 = n_c$.

If we apply the rules to generate null values for the first time, their number is
bounded by $n_r\times \ell\times {v_0}^\ell$ as we have $n_r$ rules, each rule has at most
$\ell$ existential variables, and we have at most $\ell$ frontier variables that
generate a null value.

So $v_1$ will satisfy
\begin{linenomath}
  $$v_1 \leq v_0 + n_r\times \ell\times {v_0}^\ell \leq n_r'\times
  \ell\times {v_0}^\ell,$$
\end{linenomath}
where $n_r' = n_r + 1$.

Now we just iterate to obtain $v_2$ in a similar way, and we get
\begin{linenomath}
$$
  v_2 \leq  v_1 + n_r\times\ell\times {v_1}^\ell  \leq (n_r'\times
  \ell)^{\ell+1}\times {v_0}^{\ell^2}.
$$
\end{linenomath}
\noindent If we continue this, we can get the form

\begin{linenomath}$$ v_i \leq   (n_r'\times \ell)^{\ell^i}\times {v_0}^{\ell^i} =
(n_r'\times \ell\times v_0)^{\ell^i}.
$$\end{linenomath}

\noindent That is, we get a value that for $i=s$ fulfills

\begin{linenomath}
$$v_s \leq (n_r'\times \ell\times v_0)^{\ell^s} = ((n_r+1)\times \ell\times v_0)^{\ell^s}.$$
\end{linenomath}

Now $\ell$ and $v_0=n_c$ are polynomial in the size of $P$ and $D$; hence
the terms $\ell^{\ell^s}$ and $v_0^{\ell^s}$ are double exponential in
the size of $P$ and $D$. The number $n_r$ of rules in $P'$ obeys
\begin{linenomath}$$
n_r < k\,{\times}\,|P|\,{\times}\,|\timeline|^\ell, \quad \text{ thus
} \quad n_r' \leq k\,{\times}\,|P|\,{\times}\,|\timeline|^\ell
$$
\end{linenomath}
for some constant $k$ (each rule in $P$ induces one rule
in $\rewrite{P}$ and assuming $P$ is nonempty $k$ accounts for the
extra rules  (\ref{rul_top}--\ref{rul_boxexpandat})), and thus is single exponential in the size of $P$
and $D$. Consequently,
\begin{equation}
\label{eq:tlwa-n_r}
{n'_r}^{\ell^s} \leq (k\,{\times}\,|P|\,{\times}\,|\timeline|^\ell)^{\ell^s} = (k\,{\times}\,|P|)^{\ell^s}\,{\times}\,|\timeline|^{\ell^{s+1}}
\end{equation}
is double exponential in the size of $P$ and $D$; this shows the claim.
\end{proof}

We note that under data complexity, the bound in
Lemma~\ref{lem:tlwa-nulls-bound} is still double exponential
in the length $|\timeline|$ of the timeline $\timeline$, and in fact instances where double
exponentially many null values are created do exist (see the proof of
Lemma~\ref{lem:complexity-tlwa-upper}). However, if in addition $|\timeline|$ is bounded by a constant, then the bound is
polynomial. Hence, BCQ answering with
programs in $\bcqclass{\CL{TLWA}(\timeline)}$ is tractable in this case.


Based on Lemma~\ref{lem:tlwa-nulls-bound}, we obtain the following result.


\begin{lemma}
\label{lem:complexity-tlwa-upper}
Problem $\bcqclass{\CL{TLWA}(\timeline)}$ is in \doubleexp{} under
data and combined complexity.
\end{lemma}

\begin{proof} Without loss of generality, the instance is trimmed.
We note that the program $\nowindows{P}$ is a logical strengthening of
$P$, as in each rule $r$ in $P$ the body of $r$ is weakened; that is,
more rule applications to derive null-free facts are possible for $\nowindows{P}$ over
$D$ than for $P$. Hence, an upper bound
for deriving the query atom $q$ with  $\nowindows{P}$ over $D$ will give us an upper bound for
deriving $q$ with  $P$ over $D$ as well. In the sequel, we thus
consider $\nowindows{P}$ and use Theorem~\ref{theo_tempground_termination}.

By Lemma~\ref{lem:tlwa-nulls-bound},
the number of abstract constants and nulls created by evaluating
the program $P' = \ground_\timeline(\nowindows{P})$ over
$\rewrite{D}\cup \rewrite{T}$ is bounded by a double exponential
number
$b = (k\,{\times}\,|P|\,{\times}\,|\timeline|^\ell\,{\times}\,\ell\,{\times}\,n_c)^{\ell^s}$.
Recall again that rules \eqref{rul_boxinit}--\eqref{rul_boxexpandat} are omitted from $P'$,
as remarked before.

By Lemma~\ref{lemma_noTime_correct}, instead of $P'$ we can equivalently
consider $\notime{P'}$, in which the number of predicates is bounded by  $\|P\|\cdot|\timeline|$
and their arities are bounded by the maximal rule length $\ell$ in
$P$. Hence, no more than
\begin{linenomath}\begin{eqnarray*}
\label{eq:bound-atoms}
c_1\|P\|\cdot|\timeline|b^{\ell}
                 &\!\!\!=\!\!\! & c_1\|P\|\cdot|\timeline|((k\,{\times}\,|P|\,{\times}\,|\timeline|^\ell\,{\times}\,\ell\,{\times}\,n_c)^{\ell^s})^{\ell} \nonumber \\
                 &\!\!\!=\!\!\!& c_1\|P\|\cdot|\timeline|(k\,{\times}\,|P|\,{\times}\,|\timeline|^\ell\,{\times}\,\ell\,{\times}\,n_c))^{\ell{\times}\ell^s}
\end{eqnarray*}\end{linenomath}
many ground atoms, where $c_1$ is a constant, will be
derived by the skolem chase to answer the query $q$,  which is double exponential in the size of $P$
and $D$. As each derivation step can be done, relative
to the already derived atoms, in time exponential in the maximal
number of variables in a rule (by simple rule matching) and thus in double exponential time,
the overall time to run the skolem chase is bounded double exponentially.
Furthermore, computing $P'$ is feasible in exponential time in the
size of $P$ and $D$, and computing $\rewrite{D} \cup \rewrite{T}$ is feasible
in polynomial time in the size of $P$ and $D$. Summing up,
this yields a double exponential time upper bound for problem
$\bcqclass{\CL{TLWA}(\timeline)}$.
\end{proof}

\begin{lemma}
\label{ref:tlwa-2exptime-hardness}
Problem $\bcqclass{\CL{TLWA}(\timeline)}$ is \doubleexp-hard under
data complexity and combined complexity.
\end{lemma}

\begin{proof}
This result can be accomplished by adjusting a \doubleexp-hardness
proof for BCQ Answering from a set of WA
existential rules by \citeauthor{cali_query_2010} (\citeyear{cali_query_2010}). Their
proof presents an encoding of the acceptance problem for a
(deterministic) Turing machine $M$ that operates on a given input $I$ in double exponential
time. At the core of the encoding are existential rules that generate double
exponentially many null values with a linear order (given by a
successor relation $succ$ and predicates $min$ and $max$ that single
out the first and the last element, respectively). The rules are
schematic and use indexed predicates $succ_i, min_i, max_i, r_i, s_i$
which are defined inductively for $i=0,\ldots m$.

The machine computation is then simulated using standard Datalog
rules at $m$,  which are fixed (independent of the machine
$M$); further rules serve to describe the tape contents of the initial
configuration. For a machine description $M$, a Boolean query $q =
accept(X)$, where $X$ refers to a time instant of the computation
(represented by a null), evaluates to true iff $M$ accepts the input.

We now describe the construction following \cite{pieris2011ontological},
adapted to our needs for $\CL{TLWA}(\timeline)$ programs. The key observation is that each indexed predicate
$p_i$ from above can be replaced by a predicate  $p$, such that
$p_i(\vec{x})$ is represented by $@_i p(\vec{x})$, where
we use a timeline $\timeline = [0,\ldots,m]$.

Let $M = \tuple{S, \Lambda, \blank, \delta, s_0 , F}$ be an
(one-tape) deterministic Turing machine (DTM), where $S$ is a finite (non-empty) set of
states, $\Lambda$ is the finite (non-empty) set of the tape symbols,
$\blank\in \Lambda$ is the blank symbol, $\delta: (S \setminus F) \,{\times}\ \Lambda \rightarrow S {\times} \Lambda {\times} \{-1,1,0\}$ is the transition function, $s_0 \in S$ is the initial state, and
$F \subseteq S$ is the set of accepting states. We assume that M is well-behaved
and never tries to read beyond its tape boundaries.

Without loss of generality, we can always assume that $M$ has exactly
one accepting state, denoted as $s_{acc}$, and that $s_0$ and
$s_{acc}$ are always the same (i.e., fixed). Furthermore, we may assume
that $M$ operates on empty input ($I$ is void; we could in polynomial time construct a machine $M'$ that first
writes $I$ on the tape and then simulates $M$).
This assumption is not made in \cite{pieris2011ontological}, but simplifies the construction.

We construct a fixed $\CL{TLWA}(\timeline)$ program $P$, a data stream $D$, and a
(fixed) BCQ $q$ such that $P,D \models q$  iff $M$ accepts the empty
input $I$ within time $2^{2^m}$-1 where $m=n^k$, $k>0$, and $n$ is the
size of $M$; here $\delta$ is represented by a table $T_\delta$ that
holds tuples $\vec{t} = (s, a, s', a', d)$, with the meaning that if the
machine reads $a$ in state $s$ at position $k$ on the tape, then it replaces
$a$ with $a'$, changes to state $s'$, and moves the cursor
to position $k+d$.

We use the following predicates:

\begin{itemize}
\item $symbol/3$ to hold the contents of a cell of the tape, where
$symbol(\tau,\pi,a)$ means that at time instant $\tau$, cell $\pi$
holds symbol $a$;
\item $cursor/2$ to hold the position of the cursor (read/write head),
where $cursor(\tau,\pi)$ means that at time instant $\tau$, the cursor
is at position $\pi$;
\item $state/2$ to  hold the state, where $state(\tau,s)$ means that
at time instant $\tau$, the machine is in state $s$;

 \item $transition/5$ to store the transition function $\delta$; for
  each tuple $\vec{t} \in T_\delta$, we have an atom
  $transition(\vec{t})$. This predicate is extensional, i.e., it is in the
  data stream.

\item  $accept/1$ to hold that $M$ accepts, where
    $accept(\tau)$ means that it accepts at time instant $\tau$;

\item  $succ/2$, $min/1$, $max/1$, $r/1$, $s/3$: these predicates are
        auxiliary predicates to generate  double exponentially many
        symbols for the simulation of $M$;

\item  ${\leq}/2$: this is a linear order on the (double exponentially many)
       elements of $r$ at the end of the stream;

\item  $end/0$: an atom to mark the end of the stream.

\end{itemize}

The idea is that in the stream, at time point 0 we will have $2^{2^0}=2$
many elements, $c_0$ and $c_1$, which are distinct constants. Using $r$ and
$s$, these constants will create new nulls in a progressive
fashion along the timeline, such that at time point $m$, we shall
have $2^{2^m}$ many elements.

The data stream $D = (\timeline,v_D)$ has the timeline $\timeline = [0,m]$, and we
put at time $m$ the description of the transition function of $M$,
i.e., all facts $transition(s, a, s', a', d)$, and the atom $end$; there are no further
atoms in $D$. That is, $v_D = \{ m \mapsto \{
transition(\vec{t}) \mid \vec{t} \in T_\delta\} \cup \{ end\}$.

The program $P$ consists of facts and rules as follows.
\begin{itemize}
  \item $@_0 min(c_0)$, $@_0 max(c_0)$, $@_0 succ(c_0,c_1)$, $@_0
               r(c_0)$, $@_0 r(c_1)$.

 \item (initialization rules) the tape of $M$ will be initialized to
       all blanks (owing to our assumption; recall that rules are implicitly universally quantified over time):
       \begin{linenomath}
$$end, min(X), r(Y)  \to symbol(X,Y,\blank);$$
\end{linenomath}
the cursor is at the initial position:
       \begin{linenomath}
$$end, min(X) \to cursor(X,X)$$
\end{linenomath}
and the machine is in the
                initial state:
       \begin{linenomath}
$$end, min(X) \to state(X,s_0)$$
\end{linenomath}
 \item (transition rules) three rules describe the moves of $M$:
 \begin{itemize}
 \item left move:
       \begin{linenomath}
\begin{align*}
&end, transition(S_1,A_1 , S_2 , A_2 ,-1), \\
& symbol(T_1,C_2,A_1), state(T_1,S_1), \\
& cursor(T_1,C_2), succ(T_1, T_2), succ(C_1,C_2 ) \to \\
&symbol(T_2,C_2,A_2),state(T_2,S_2), cursor(T_2,C_1)
\end{align*}
\end{linenomath}

 \item right move:
       \begin{linenomath}
\begin{align*}
&end, transition(S_1,A_1 , S_2 , A_2 ,1), \\
& symbol(T_1,C_1,A_1),state(T_1,S_1), \\
& cursor(T_1,C_1), succ(T_1, T_2), succ(C_1,C_2 ) \to \\
&symbol(T_2,C_1,A_2),state(T_2,S_2), cursor(T_2,C_2)
\end{align*}
        \end{linenomath}

\item  stay move:
       \begin{linenomath}
\begin{align*}
&end, transition(S_1,A_1 , S_2 , A_2 ,0), symbol(T_1,C,A_1), \\
&state(T_1,S_1), cursor(T_1,C), succ(T_1, T_2) \to \\
&symbol(T_2,C,A_2),state(T_2,S_2), cursor(T_2,C)
\end{align*}
        \end{linenomath}

\end{itemize}

\item (inertia rules) the contents of the tape not at the cursor
position has to be carried over:
       \begin{linenomath}
\begin{align*}
& end, cursor(T_1,C_2), succ(C,C_2),  {\leq}(C_1,C),\\
&  symbol(T_1,C_1,A), succ(T_1,T_2) \to \\
& symbol(T_2,C_1,A)\\[3pt]
& end, cursor(T_1,C_1), succ(C_1,C),  {\leq}(C,C_2), \\
&  symbol(T_1,C_2,A), succ(T_1,T_2) \to \\
& symbol(T_2,C_2,A)
\end{align*}
        \end{linenomath}

\item (acceptance rule)
       \begin{linenomath}
$$
end, state(T,s_{acc}) \to accept(T)
$$
        \end{linenomath}
\end{itemize}

In addition to these rules, we have rules that define the auxiliary
predicates. At the heart is the generation of nulls at the
time point $m$, which represent exponentially long bit vectors. This
is accomplished with the following rules:
       \begin{linenomath}
\begin{align*}
&   r(X), r(Y) \to \exists Z. s(X,Y,Z) \\[3pt]
&   s(X,Y,Z), T'=\mathsf{now}+1  \to  \exists Z. @_{T'} r(Z)\\[3pt]
&   s(X,Y_1,Z_1), s(X,Y_2,Z_2),\\
&  succ(Y_1,Y_2), T'=\mathsf{now}+1  \to  @_{T'} succ(Z_1,Z_2)\\[3pt]
&   s(X_1,Y_1,Z_1), s(X_2,Y_2,Z_2), max(Y_1),
    min(Y_2), \\
& succ(X_1,X_2), T'= \mathsf{now}+1  \to @_{T'} succ(Z_1,Z_2)\\[3pt]
&   s(X,X,Z), min(X), T'=\mathsf{now}+1 \to @_{T'} min(Z)\\[3pt]
&   s(X,X,Z), max(X), T'=\mathsf{now}+1 \to @_{T'} max(Z)
\end{align*}
        \end{linenomath}
Intuitively, the effect of these rules is that, at each time point $i$,
the elements in $r$ are paired, such
that tuples of length $2^i$ yield tuples of length
$2^{i+1}$, where nulls give names to these tuples.

Finally, the order $\leq$ is defined as follows:
       \begin{linenomath}
\begin{align*}
& end, r(X) \to {\leq}(X,X) \\
& end, succ(X,Y), {\leq}(Y,Z) \to {\leq}(X,Z).
\end{align*}
        \end{linenomath}
This completes the description of the program $P$. Notice that for
varying inputs $M$, the rules of $P$ are by our assumptions the same, thus fixed.
Furthermore, only a single rule has an existential variable in the head; however, no cycles
through special edges in the dependency graph of the temporal
grounding of the program $P$ over $D$ are possible.

It can be shown that some atom $accept(t)$, where $t$ is a ground term
(in fact, a null) can be derived with $P$ over the data stream $D$ at time point $m$ iff $M$ accepts (the empty) input.

To complete the construction, we thus set the BCQ $q$ to $accept(X)$
for evaluation at time point $m$. Alternatively, we could also
introduce a rule
       \begin{linenomath}
$$
accept(X) \to q
$$
        \end{linenomath}
and answer $q$ at time point $m$.
\end{proof}

Please notice that in the construction in
Lemma~\ref{ref:tlwa-2exptime-hardness}, we may change the accept rule
to
       \begin{linenomath}
$$
accept(X) \to @_0 q
$$
        \end{linenomath}
i.e., put the query atom at time point 0; thus asking whether $q$ can
be entailed with $P$ over $D$ at time point 0 is \doubleexp-hard,
i.e., for BCQ at a fixed query time.

\section{Further experimental details}
\label{app:experiments}

The experiments were conducted using a MacbookPro16,1 with Intel Core i7 2.6GHz
and 32GB RAM. Scenario $S_A$ is meant to simulate a stream with $b$ conveyor
belts and multiple sensors that measure the speed and the temperature. When the
speed is too slow or the temperature is too high, the rules trigger warnings and
errors.  The rules used in this scenario are the following:
\begin{linenomath}\begin{align}
    \pbelt(X) &\rightarrow \exists Y. \poperator(X,Y) 
    \label{r:operator} \\
    \timeWindow{5} \Diamond \pspeed(X,Y) \land \plowspeed(Y) & \rightarrow
    \exists Z.\pbrokengear(X,Z) \label{r:brokengear} \\
\timeWindow{3} \Box \ptemp(X,Y) \land \phightemp(Y) &\rightarrow \exists
Z.\pproblemId(Z,X) \label{r:newproblem} \\
\pproblemId(Y,X) \land \poperator(X,Z) &\rightarrow \passign(Y,Z)
\label{r:assign} \\
\timeWindow{3} \Box \pproblemId(Z,X) &\rightarrow \perror(X) \label{r:error}
\end{align} \end{linenomath} \noindent where \poperator{}, \pbrokengear{},
\pproblemId{} is short for \poperatorFull{}, \pbrokengearFull{},
\pproblemIdFull{}, respectively.
Here (\ref{r:operator}) and (\ref{r:brokengear}) use existentials to introduce
new potentially unknown individuals while
(\ref{r:newproblem}) introduces a new incident ID if the temperature is high;
(\ref{r:assign}) assigns the incident to the current belt's operator, while
(\ref{r:error}) blocks the belt if the incident has been persisting since three
time points in the past.

Scenario $S_B$ is created as follows.
$\deep$ from the $\chasebench$ suite \cite{bench_chase} contains 1k facts
and 1.1k existential rules with 1 body atom and 3-4 head atoms, and predicates
having arity 3 or~4. We created a stream by
copying all 1k facts on each time point, and prefixed in rules each body atom
$B$ in with either $\window^n \Box$ (50\%) or $\window^n \Diamond$ (50\%) for
some $n$. In both scenarios, we created streams with 100 time points and set
$\ell\,{=}\,6$, which is large enough to fill all windows.

The data stream contains $b=100$ belts. For each belt, the data stream contains
3 facts: the identifier of the belt (e.g., $\pbelt(b_1)$, the value of the
speed, which can be either high or low (e.g., $\pspeed(b_1,low)$), and the value
of the temperature, which is an integer from 1 to 9 (e.g., $\ptemp(b_1,3)$). At
each time point, every belt has a slow speed with probability $p_1$ (hence
triggering rule~(\ref{r:brokengear})), and high temperature with probability
$p_2$, which lasts for at least four consecutive time points with probability
$p_3$ to trigger rules~(\ref{r:newproblem}-\ref{r:error}). The computation of
the \elars{} model is invoked at each time point and the reported numbers contain
the averages across all time points.

}{}

\end{document}